\documentclass[sigconf]{acmart}
\usepackage{algorithm}
\usepackage{algpseudocode}
\usepackage{comment}
\usepackage{multirow}

\AtBeginDocument{%
  \providecommand\BibTeX{{%
    \normalfont B\kern-0.5em{\scshape i\kern-0.25em b}\kern-0.8em\TeX}}}

\setcopyright{acmlicensed}
\copyrightyear{2024}
\acmYear{2024}
\acmDOI{XXXXXXX.XXXXXXX}

\acmConference[ACM Conference'24]{}{August 2024}{XXX, XXXX}
\acmISBN{978-1-4503-XXXX-X/18/06}

\newcommand{\name}{{\sc Binder}}

\begin{document}

\title{{\sc Binder}: Hierarchical Concept Representation through Order Embedding of Binary Vectors}

\author{Croix Gyurek}
\affiliation{%
  \institution{University of Waterloo }
  \streetaddress{200 University Avenue West Waterloo}
  \state{Ontario}
  \country{CA}
  \postcode{N2L 3G1}
}
\authornote{Both authors contributed equally to this research.}
\email{cgyurek@uwaterloo.ca}

\orcid{1234-5678-9012}
\author{Niloy Talukder}
\authornotemark[1]
\affiliation{%
  \institution{Indiana University Purdue University Indianapolis}
  \streetaddress{420 University Blvd}
  \city{Indianapolis}
  \state{Indiana}
  \country{USA}
  \postcode{46202}
}
\email{ntalukde@iupui.edu}

\author{Mohammad Al Hasan}
\affiliation{%
  \institution{Indiana University at Indianapolis}
  \streetaddress{420 University Blvd}
  \city{Indianapolis}
  \state{IN}
  \country{USA}
  \postcode{46202}
}
\email{alhasan@iu.edu}

\renewcommand{\shortauthors}{Talukder and Gyurek, et al.}

\renewcommand{\shortauthors}{Gyurek and Talukder, et al.}

\begin{abstract}
For natural language understanding and generation, embedding concepts using an 
order-based representation is an essential task. Unlike traditional point vector based
representation, an order-based representation imposes geometric constraints on the
representation vectors for explicitly capturing various semantic relationships that may exist
between a pair of concepts. In existing literature, several approaches on order-based 
embedding have been proposed, mostly focusing on capturing hierarchical relationships; examples include vectors in Euclidean space, complex, Hyperbolic, order, and Box Embedding. 
Box embedding creates region-based rich representation of concepts,
but along the process it sacrifices simplicity, requiring a custom-made optimization scheme 
for learning the representation. Hyperbolic embedding improves embedding quality by 
exploiting the ever-expanding property of Hyperbolic space, but it also suffers from the
same fate as box embedding as gradient descent like optimization is not simple in the
Hyperbolic space. In this work, we propose \name, a novel approach for order-based 
representation. {\name } uses binary vectors for embedding, so the embedding vectors are compact with an order of magnitude smaller footprint than other methods. 
{\name } uses a simple and efficient optimization scheme for learning representation
vectors with a linear time complexity. Our comprehensive experimental results show that {\name } is very accurate, yielding competitive results on the representation task. But {\name } stands out from its competitors on the transitive closure link prediction task as it can learn concept embeddings just from the direct edges, whereas all existing order-based approaches rely on the indirect edges. For example, {\name } achieves a whopping \textbf{70\%} higher $F1$-score than the second best method (98.6\% vs 29\%) in our largest dataset, WordNet Nouns (743,241 edges), in the 0\% transitive closure experiment setup.

\end{abstract}

\keywords{Concept Graph, Hierarchical Embedding, Order Embedding, Binary Vector Embedding}

\maketitle

\section{Introduction}
Ontologies transcribe the knowledge of a domain through formal listing of concepts along with the knowledge of various semantic relations that may exist among those concepts.
The most important among these relations is hypernym-hyponym, which captures the {\em is-a} relation between a specialized and a general concept of a domain. 
Such knowledge are essential for achieving optimal results in various natural language generation tasks, such as image caption generation \cite{karpathy_deep_2015, vinyals_show_2015}, question-answering \cite{dai_cfo_2016, ren_query2box_2020, yao_kg-bert_2019}, and taxonomy generation \cite{nakashole_patty_2012, wu_probase_2012}.
For example, using an appropriate ontology, an image caption generator can opt for a generic text ``a person is walking a dog'' instead of a more informative text, ``a woman is walking her dog'', if the model is not very confident about the gender of the person.
However, building a large conceptual ontology for a domain is a difficult task requiring gigantic human effort, so sophisticated machine learning models, which can predict \emph{is-a} links between pairs of concepts in an ontology, is of high demand.

Predicting is-a relationship between two concepts in an ontology can be viewed as a link prediction task in a concept graph or an ontology chart. Although link prediction \cite{hasan_survey_2011} is a well-studied task, predicting links in a concept graph did not receive much attention by the NLP or Machine Learning researchers. The closest works that exist are on predicting links
in a knowledge graph~\cite{galarraga_amie_2013, lao_random_2011}, but such works are customized for the Resource Description Framework (RDF) type of data. Lately, 
node embedding using shallow~\cite{perozzi_deepwalk_2014, tang_line_2015, grover_node2vec_2016} or
deep~\cite{guo_learning_2019, neelakantan_compositional_2015, schlichtkrull_modeling_2017, shang_end--end_2018, nguyen_novel_2018, yao_kg-bert_2019} neural networks have shown improved performance for solving link prediction tasks. Most of these works consider undirected graphs, so they are not suitable 
for asymmetric concept graphs. Some node embedding works~\cite{ bordes_translating_2013, he_learning_2015, lin_learning_2015, sun_rotate_2019, trouillon_complex_2016,  wang_knowledge_2014} also embed knowledge graphs, where embedding of head, tail, and relation nodes are learned, but such works, again, are only suitable for RDF based data representation.

In recent years, some works have been proposed for embedding concepts in such a way that the embedding
vectors would capture semantic relationships among the concepts. One of the earliest efforts in this direction is \cite{vendrov_order-embeddings_2015, lai_learning_2017}, which proposed Order Embedding. The main idea is to embed
the concepts in the positive cone $\mathbb{R}_+^d$. In this embedding, 
if $a$ is-a $b$, then their corresponding embedding vectors satisfy $f(b) \le f(a)$ element-wise. 
In this way, a generic concept hovers closer to the origin with a smaller norm
than the associated specialized concept. In another line of works~\cite{nickel_poincare_2017, ganea_hyperbolic_2018}, which focus on
embedding trees, DAGs or tree-like graphs, hyperbolic space is used instead of Euclidean
space. In Hyperbolic space, two non-intersecting straight lines diverge, 
allowing for the embedding of more objects along the periphery. It is beneficial for embedding
a tree structure which has exponentially more nodes at a higher depth. The third line of work, known
as box embedding \cite{vilnis_probabilistic_2018, li_smoothing_2019, dasgupta_improving_2020,dasgupta_box--box_2021, boratko_capacity_2021}, deviates from vector (point) embedding, and instead uses a rectangular region 
for embedding a concept. This representation is richer as it both helps embedding
order relation and part-of relations between concepts, or overlapping concepts, thereby overcomes some of the
limitations of earlier approaches. Other existing approaches for embedding generic graph-based relational data use simple \cite{bhattacharjee_what_2023}, bilinear \cite{nickel_three-way_2011}, or complex \cite{trouillon_complex_2016} vectors and make use of inner product and distance metric in Euclidean space
to capture relations. 

In this work, we propose an order embedding model that is simple, elegant and compact. 
Our idea is to use binary vectors for order embedding, i.e. for each entity $a$, $f(a) = \mathbf{a} \in \{0, 1\}^d$. 
In other words, we embed each object at a vertex of a $d$-dimensional non-negative unit hypercube, 
where $d$ 
is a user-defined parameter. The overall approach is simple, as for denoting $a$ is-a $b$, we require that $f(b)_j = 1 \!\implies\! f(a)_j = 1, \forall j \in [1 \!:\! d]$, i.e., along any embedding 
dimension $j$, 
if $\mathbf{b}_j$ is 1, then $\mathbf{a}_j$ must be 1. The idea is fairly intuitive; if we consider each
dimension denoting some latent property which make something $b$, given that $a$ is-a $b$, $a$ 
also have those properties. Since it uses bits for embedding, the embedding vectors
are compact with an order-of-magnitude smaller memory footprint (see Section ~\ref{subsec:resource_consumptions}) compared to other methods.
Our embedding idea is elegant as it captures the is-a relation through intent-extent philosophy of formal concept analysis (FCA) \cite{ganter_formal_1999}, which is a principal way of deriving a concept hierarchy.

The major challenge for our proposed embedding idea is finding an effective optimization 
algorithm for learning the embedding, as we deviated away from continuous Euclidean space and 
embraced a combinatorial space. In this sense,
given the training data (a collection of hyponym-hypernym pairs), learning
the embedding of the objects in the training data becomes a classical combinatorial feasibility
task, a known NP-complete problem. We use a randomized local search algorithm inspired by 
stochastic gradient descent for solving this problem. In each epoch, our optimization method probabilistically flips the bits of the embedding vectors, where the flipping probability is computed from
the ``gradient'' value for each bit position. The proposed optimization method is innovative and novel 
as it consists of computing a proxy of gradient for a binary space (Section ~\ref{sec:gradient-derivation}),
and then use that for computing a flip probability (Section ~\ref{sec:flip_probability}). 
Our optimization algorithm is very fast; the overall computational complexity is $O(ndT(|P| + |N|))$, for $n$ words and $d$ dimensions, which is linear in each variable.  
We name our method {\sc \name}~\footnote{The name Binder is an abbreviation of \textbf{Bin}ary Or\textbf{der} Embedding.}. We claim the following contributions:

\smallskip
\noindent\textbf{(1)} We propose \name, a novel order embedding approach which embeds entities at 
vertices of a $d$-dimensional hypercube. We show that {\name } is ideal for finding representation 
of entities or concepts with hyponym-hypernym relationships. {\name } is simple, compact, efficient, and has better generalization capacity over transitive edges than existing methods in a transductive setting.

\smallskip
\noindent\textbf{(2)} \name\ uses a novel local search based optimization algorithm for solving
the embedding learning task. The proposed algorithm is simple, efficient and effective, and
a proxy of gradient descent for the combinatorial space.

\smallskip
\noindent\textbf{(3)} Experiments on 6 benchmark datasets show that \name\ exhibits superior performance over the existing state-of-the-art algorithms on transitive closure link prediction tasks.

\section{Binary Order Embedding (\name)}

\subsection{Motivation \label{sec:motivation}}

We solve the following task: From a collection of is-a relations between a pair of entities, obtain embedding of the entities such that embedding vectors geometrically capture the order imposed through the is-a relation. For representing $x$ is-a $y$, one earlier 
work \cite{vendrov_order-embeddings_2015} has imposed order by requiring $x_i \leq y_i, \forall i \in [1:d]$ for embedding vectors $x$ and $y$ in real space. BINDER uses a similar idea, but instead it uses binary space, in which $x_i \leq y_i$ becomes $y_i \implies x_i$. Implication obeys the transitive property, so BINDER’s binary representation works well for is-a relationship, which is transitive. {\name } has the following benefits:

\smallskip
\noindent(1) Binary representations are compact, often taking order of magnitude less memory than embeddings in real space, and is computationally efficient (demonstrated in Section~\ref{subsec:resource_consumptions}).

\smallskip
\noindent(2) Binary representation can immediately provide representation of concepts that can be obtained by logical operation over the given concept vectors.
For instance, given vectors for the concepts ``vehicle'' and ``flying'', we can find a vector for instances of ``flying vehicle'' subtree by taking the union of the \emph{vehicle} and \emph{flying} vectors.
Or, if we have representation vectors for “men's shoe” and “women's shoe”, we can obtain a representation vector for shoe by taking the intersection of the above two vectors. Such operation can be extended to other complex Boolean operations. We provide elaborate discussion on \name's binary properties in Appendix ~\ref{sec:binary_properties}.

\smallskip
\noindent(3) \name\ provides explainable representation vectors in the sense that we can treat the bit indices as a set of latent properties: a “1” value at a dimension means that the entity possesses that property, and a “0” value means that it does not possess that property. A 
vector in real space does not provide such intuitive interpretation. 
We give a small demonstration of this in Figure \ref{fig:toyvisual} in Section ~\ref{subsec:explainability}, where we trained our model on a small lattice. In particular, the embedding, being a binary matrix, can be thought of as a machine-learned object-attribute table. The number of $1$'s in a concept's representation provides an indication of its narrowness. Using domain knowledge, and by observing the distribution of $1$'s in a column, one can further deduce which dimension may represent which of the properties (intent). In fact, we can 
think of \name's embedding as capturing the intent-extent philosophy of Formal Concept Analysis (FCA)~\cite{ganter_formal_1999}, which provides a principled approach for deriving a concept hierarchy from a collection of objects and their properties. 


Above all of these, our experiments show that \name\ performs competitively on the representation task and is prodigiously superior on the transitive closure link prediction tasks.

\subsection{Problem Formulation \label{sec:formulation}}

\name\ embeds each concept $a$ through a $d$-dimensional binary vector $\mathbf{a}$, so every concept is embedded at the vertex of a $d$-dimensional non-negative unit hypercube, where $d$ is a user-defined parameter. If $a$ is-a $b$, and $\mathbf{a}$ and $\mathbf{b}$ are their representation, then \name\ satisfies $\mathbf{b}_k = 1 \!\implies\! \mathbf{a}_k = 1, \forall k \in \{1, \dots, d\}$.
This idea comes from the observation that when $a$ is-a $b$, $a$ must inherit all the properties that $b$ possesses. In \name's embedding, a `1' in some representation dimension denotes ``having a latent property'': if the $j$'th bit of $\mathbf{b}$ has a value of 1, i.e., $b$ possesses property $j$, then $a$ must also possess property $j$, which is captured by the above requirement.

To learn embedding vectors for a collection of concepts in a domain, \name\ uses a supervised learning approach. Given a set of concepts $W$ and partial order concept-pairs $P = \{(a,b): a \text{ is-a } b\}$, \name's task is to find an embedding function $B: W \to \{0,1\}^d$ such that for all $a, b \in W$,
\begin{align}
(\mathbf{a} \cap \mathbf{b}) = \mathbf{b} \text{ iff } (a,b) \in P \text{ and } a \neq b \label{eq:pf}
\end{align}
holds; here $\mathbf{a} = B(a)$ and $\mathbf{b} = B(b)$ are the embedding vectors for concepts $a$ and $b$, and $\cap$ denotes the bitwise AND operation.

The above learning task is a binary constraint satisfaction problem (CSP), which is a known NP-Complete task~\cite{cormen_introduction_2022}. To solve it, \name\ uses a randomized local search algorithm, which is fast and effective. Note that given a training dataset, $P$, if \name's embedding solution satisfies all the constraints, then the embedding is perfect and all the partial order pairs in $P$ can be reconstructed from the embedding with a 100\% accuracy. But the goal of our embedding is not necessarily yielding a 100\% reconstruction accuracy on the training data, rather to perform ''is-a'' prediction task on an unseen test dataset, so we do not strive to solve the CSP task exactly. In the next section, we discuss \name's learning algorithm.

\noindent
{\bf Notations:} Italic letters such as $a,b$ denote entities, while boldface $\mathbf{a},\mathbf{b}$ denote their embedding vectors. $\mathbf{a}_j$ denotes the value at the $j$'th position in $\mathbf{a}$. In the algorithms, we use $B$ to denote the complete binary embedding matrix, $B[a,:]$ for $\mathbf{a}$, and $B[a,j]$ for bit $j$ of said vector. We use $*$ to denote element-wise multiplication; all arithmetic operations are done in $\mathbb{Z}$ or $\mathbb{R}$. Finally, we write pairs in hyponym-hypernym order: $(a,b)$ refers to the statement ``$a$ is-a $b$''.

\subsection{Training Algorithm\label{sec:training-algorithm}}

The learning task of \name\ is a CSP task, which assigns $|W|$ distinct binary $d$-bit vectors to each of 
the variables in $W$, such that each of the constraints in the partial order $P$ is satisfied. Various search strategies have
been proposed for solving CSP problems, among which local search and simulated annealing are used widely~\cite{beck_combining_2011}. For guiding the local search, we model this search problem as an optimization
problem, by designing a loss function that measures the fitness of an embedding. A simple measurement is the number of
pairs in $P$ that violates the constraint in Equation~\ref{eq:pf}. Note that the constraint is ``if and only if'', which
means for any pair $(a',b')$ that is not in $P$, we want $\mathbf{a}'$ and $\mathbf{b}'$ to \emph{not} satisfy this constraint. If $|W| = n$, we have exactly
$|P|$ constraints for the positive pairs (we call these \emph{positive constraints}), and $n^2 - n - |P|$ \emph{negative constraints} for the negative pairs.
Using these constraints, we compute a simple loss function---a linear function of the number of violated positive and negative constraints, $Loss = Loss_P + Loss_N$, as shown below:
\begin{align}
    &Loss_P = \alpha \sum_{(a,b) \in P} \sum_{j} \textbf{1}_{(\mathbf{a}_j, \mathbf{b}_j) = (0,1)}(a,b) \label{eqn:loss-p} \\
    &Loss_N = \beta \sum_{(a',b') \in N} \textbf{1}_{\forall j (\mathbf{a}'_j, \mathbf{b}'_j) \in \{(0,0), (1,0), (1,1)\}}(a',b') \label{eqn:loss-n}
\end{align}
where $\alpha$ and $\beta$ are user-defined parameters and \textbf{1} is the indicator function.
Due to the above loss function, \name's learning algorithm relies on the existence of negative pairs $N \subseteq \{(a', b'): a' \text{ is-not-a } b'\}$. If these negative pairs are not provided, we generate them by randomly corrupting the positive pairs $P$ as in \cite{vendrov_order-embeddings_2015, nickel_poincare_2017,ganea_hyperbolic_2018}, by replacing $(a,b) \in P$ with $(r,b)$ or $(a,r)$ where $r$ is sampled randomly from the entity set $W$.

For local search in \emph{continuous} machine learning, the search space is explored using some variant of gradient descent. This gradient is defined as the derivative of the loss function (or an approximation thereof) with respect to each parameter. With binary vectors, the parameters are discrete, but we can get a proxy of the ``gradient'' by taking the
\emph{finite difference} between the current value of the loss function, and the new value after a move is made.
In a continuous space, vectors can be updated by adding or subtracting a delta, but in discrete space, the new
vector comes from neighborhood that must be defined explicitly.
If the neighborhood is defined by unit Hamming distance, we can make a neighbor by flipping one bit of one vector, but for large datasets, this approach will converge very slowly.
In \name, we randomly select bits to be flipped by computing a probability from the gradient of each bit position.

{\name}'s algorithm solves a combinatorial satisfiability task, a known NP-complete problem. By computing a
gradient and then utilizing it to decide bit flipping probability, it works as a gradient descent scheme 
in the discrete space to minimize the objective function defined in Equations~\ref{eqn:loss-p} and \ref{eqn:loss-n}. Note that 
the bit flipping probability decreases gradually as the number of violated constraints decreases with each subsequent
epoch. This gives \name's learning algorithm a flavor of local search with simulated annealing. 

\subsection{Gradient Derivation\label{sec:gradient-derivation}}

\name's gradient descent scheme is based on correcting order of positive pairs by flipping bits, which are chosen randomly with a probability computed from gradient
value. Below we discuss how gradient is computed.

A sub-concept will share all the attributes (bits set to 1) of the concept, and possibly contain more attributes. 
For each positive pair $(a,b) \in P$ and each bit index $j$, we aim to avoid having $(\mathbf{a}_j, \mathbf{b}_j) = (0,1)$, since that would imply $a$ did not inherit attribute $j$ from $b$. 
On the other hand, for negative pairs $(a', b')$, we aim to create at least one bit index
with $(0, 1)$ bit pair. Suggested bit flipping or protecting operations for these requirements are
shown in the third and sixth column of Table ~\ref{tab:flip-positive-negative}; if $a$ is-a $b$, 
we do not want $(0, 1)$ configuration, hence we protect $\mathbf{b}_j$ in first row,
$\mathbf{a}_j$ in the fourth row, and flip either bit in the second row. On the other hand, for $a$
not-is-a $b$, we want a $(0,1)$ configuration. If the model currently suggests $a'$ is-a $b'$ (i.e. there is no $j$ where $\mathbf{a}'_j=0, \mathbf{b}'_j=1$), we correct this by flipping either the 
$\mathbf{a}'$ side of a $(1,1)$ pair (fourth row) or the $\mathbf{b}'$ side of a $(0,0)$ pair (first row), as shown in the sixth columns of the same table. Note that negative samples are needed to avoid trivial embeddings, such as all words being assigned the zero vector.\footnote{Simply requiring all positive pairs to be distinct is not sufficient, because the model could converge on some ``depth counting'' embedding where different subtrees reuse the same bit patterns. Instead, we rely on negative samples to separate positive pairs, since if $(a,b)$ is positive then $(b,a)$ will be negative.}

\begin{table}
  \caption{Logic truth table to flip bits in positive (first three columns) and negative (last three columns) pairs}
  \vspace{-3mm}
  \label{tab:flip-positive-negative}
  \begin{tabular}{cc|l||cc|l}
    \toprule
    $\mathbf{a}_j$&$\mathbf{b}_j$ & $a$ is-a $b$ & $\mathbf{a}_j'$&$\mathbf{b}_j'$ & $a'$ is-not-a $b'$ \\
    \midrule
    $0$ & $0$ & Protect $\mathbf{b}_j$ & 
        0 & 0 & Flip $\mathbf{b}_j'$ \\
    $0$ & $1$ & Flip either bit &
        0 & 1 & Protect both bits \\
    $1$ & $0$ & Don’t care &
        1 & 0 & Flip both $\mathbf{a}_j'$ and $\mathbf{b}_j'$ \\
    $1$ & $1$ & Protect $\mathbf{a}_j$ &
        1 & 1 & Flip $\mathbf{a}_j'$ \\
  \bottomrule
\end{tabular}
\end{table}

We first derive the gradient of $Loss_P$ with respect to $\mathbf{a}_j$ and $\mathbf{b}_j$.
We define the $gradient$ $\Delta_{\mathbf{a}_j}$ to be positive if flipping bit $\mathbf{a}_j$ \emph{improves} the solution according to our loss function, regardless of whether bit $\mathbf{a}_j$ is currently 0 or 1. 
As shown in Column 3 of Table~\ref{tab:grad-positive}, flipping $\mathbf{a}_j$ makes 
no change in loss for the first and third rows; but for the second row, one
violation is removed, and for the fourth row, one new violation is added. For the four 
bit configurations, $\{(0, 0), (0, 1), (1, 0), (1, 1)\}$, the improvement to $Loss_P$
is 0, 1, 0, and -1, respectively, as shown in Column 3 of Table~\ref{tab:grad-positive}. 
Column 4 of the same table shows the value of $\Delta_{\mathbf{b}_j} Loss_P$ calculated similarly.
It is easy to see that these two sets of gradients values can be written as 
$\mathbf{b}_j(1-2\mathbf{a}_j)$ and  $(1-\mathbf{a}_j)(2\mathbf{b}_j-1)$, respectively. 
Summing this over all positive pairs, $(a, b) \in P$, we get the positive gradient vectors  
\begin{align}
    \Delta_{\mathbf{a}} Loss_P &= \alpha \sum_{b: (a, b) \in P} \mathbf{b} * (1-2\mathbf{a}) \label{eqn:grad-pos-a} \\
    \Delta_{\mathbf{b}} Loss_P &= \alpha \sum_{a: (a, b) \in P} (1-\mathbf{a}) * (2\mathbf{b}-1) \label{eqn:grad-pos-b}
\end{align}

\begin{table}
\vspace{-1mm}
  \caption{Logic truth table to calculate positive loss gradient}
  \vspace{-3mm}
  \label{tab:grad-positive}
  \begin{tabular}{ccccl}
    \toprule
    $\mathbf{a}_j$ & $\mathbf{b}_j$ & $\Delta_{\mathbf{a}_j} Loss_P$ & $\Delta_{\mathbf{b}_j} Loss_P$ & Comments\\
    \midrule
    $0$ & $0$ & $0$ & $-1$ & Protect $\mathbf{b}_j$\\
    $0$ & $1$ & $1$ & $1$ & Flip either (or both) bit\\
    $1$ & $0$ & $0$ & $0$ & Don’t care\\
    $1$ & $1$ & $-1$ & $0$ & Protect $\mathbf{a}_j$\\
  \bottomrule
\end{tabular}
\vspace{-0.2in}
\end{table}

For negative loss, we count the number of $(\mathbf{a}_j' , \mathbf{b}_j' ) = (0,1)$ ``good'' bit pairs in a negative pair. If $G_{a',b'} := \big|\{j: (\mathbf{a}_j', \mathbf{b}_j') = (0,1) \}\big| = 0$, then $(a',b')$ is a false positive, so we need to flip a bit to make $(\mathbf{a}_j' , \mathbf{b}_j') = (0,1)$.
In this case, based on the bit values 
in Columns 3 and 4 of Table ~\ref{tab:grad-negative}, we derived the following algebraic expressions:
\begin{align}
   \Delta_{\mathbf{a}'} Loss_N  &= \beta \!\!\!
        \sum_{\substack{b': \, (a',b') \in N, \\ G=0}} \!\!\! \mathbf{a'} * \mathbf{b'} \label{eqn:grad-neg-a} \\
   \Delta_{\mathbf{b}'} Loss_N &= \beta \!\!\! \sum_{\substack{a': \, 
        (a',b') \in N, \\ G=0}} \!\!\! (1  -  \mathbf{a'}) * (1 - \mathbf{b'}) \label{eqn:grad-neg-b}
\end{align}

\begin{table}[ht]
  \vspace{-3mm}
  \caption{Logic truth table to calculate negative loss gradient}

  \vspace{-3mm}
  \label{tab:grad-negative}
  \begin{tabular}{ccccl}
    \toprule
    $\mathbf{a}_j'$ & $\mathbf{b}_j'$ & $\Delta_{\mathbf{a}_j'} Loss_N$ & $\Delta_{\mathbf{b}_j'} Loss_N$ & Comments\\
    \midrule
    $0$ & $0$ & $0$ & $1$ & Flip $\mathbf{b}_j'$\\
    $0$ & $1$ & $0$ & $0$ & Don’t care\\
    $1$ & $0$ & $0$ & $0$ & Don’t care\\
    $1$ & $1$ & $1$ & $0$ & Flip $\mathbf{a}_j'$\\
  \bottomrule
\end{tabular}
\vspace{-2mm}
\end{table}
\smallskip

On the other hand, if $G_{a',b'} = 1$, there is no violation, but the not-is-a
relation is enforced by only one bit, so that bit must be protected. That is, if 
$(\mathbf{a}_j' , \mathbf{b}_j' ) = (0,1)$ for exactly one $j$, then we want the gradient to be $-1$ for that index $j$ (recall that negative gradient means flipping is bad), and zero gradient for all other indices. This is true for exactly the bit $j$ satisfying $b_j'(1 - a_j') = 1$, so we have

\begin{align}
    \Delta_{\mathbf{a}'} Loss_N = -\beta \!\!\! \sum_{\substack{b': \, (a',b') \in N, \\ G=1}} \!\!\! \mathbf{b}' * (1 - \mathbf{a}') \label{eqn:grad-neg-a-j} \\
    \Delta_{\mathbf{b}'} Loss_N = -\beta \!\!\! \sum_{\substack{a': \, (a',b') \in N, \\ G=1}} \!\!\! \mathbf{b}' * (1 - \mathbf{a}') \label{eqn:grad-neg-b-j} 
\end{align}
If $G_{a',b'} \!>\! 1$, no bits need protection, so we set those gradients to 0.

Finally, we add the gradients $Loss_P$ and $Loss_N$ to get the final gradient matrix, $\Delta = \Delta^+ + \Delta^-$. The gradient computation process is summarized in Algorithm \ref{alg:gradient} in Appendix \ref{sec:pseudo-code}. In this Algorithm,
$\Delta^+$ is the gradient of $Loss_P$ and $\Delta^-$ is the gradient of $Loss_N$. $\Delta$, $\Delta^+$ and $\Delta^-$ all 
are integer-valued matrices of size $n \times d$, where $n$ is the vocabulary size and $d$ is the embedding dimension. The overall training algorithm is summarized in Algorithm \ref{alg:training} in Appendix \ref{sec:pseudo-code}.
Its computational complexity is $O(ndT|P \cup N|)$; if $N$ is sampled randomly, this is linear in each variable. The analysis of algorithm computational complexity is shown in Appendix ~\ref{sec:computational_complexity}.

\subsection{Flip probability \label{sec:flip_probability}}
In binary embedding, each bit position takes only two values, 0 and 1. Traditional gradient descent, which updates 
a variable by moving towards the opposite of gradient, does not apply here. Instead, we utilize randomness in the update step: bit $j$ of word $w$ is flipped with a probability based on its gradient, $\Delta[w,j]$. To calculate the flip probability we used $\tanh$ function. For each word $w$ and each bit $j$ we compute the gradient $\Delta[w,j]$ as in Algorithm \ref{alg:gradient} and output
\begin{equation}
    \mathrm{FlipProb}[w,j] = \max\left\{ 0, \tfrac 1 2 \tanh\left(2(r_\ell \, \Delta[w,j] + b_\ell)\right) \right\}
\end{equation}
where the learning rate $r_\ell$ controls the frequency of bit flips, and the learning bias $b_\ell$ allows the algorithm to flip bits with zero gradient (to avoid local optima).
The division by 2 prevents the model from flipping (on average) more than half of the bits in any iteration; without it, the model would sometimes flip nearly every bit of a vector, causing it to oscillate.
Therefore, we (softly) bound probabilities by $\frac 1 2$ to maximize the covariance of the flips of each pair of vectors.
The inside multiplication by 2 is because $\frac 1 2\tanh(2p) \approx p$ for small $p$, so hyperparameter selection is more intuitive: $b_\ell$ is (almost exactly) the probability of flipping a neutral bit, and $\alpha r_\ell$ approximates the increase in probability for each positive sample that would be improved by flipping a given bit (and likewise for $\beta r_\ell$ and negative samples).
We note that the three hyper parameters $r_\ell$, $\alpha$, and $\beta$ are somewhat redundant, since we can remove $r_\ell$ and replace $\alpha$ and $\beta$ with $r_\ell\alpha$ and $r_\ell\beta$ respectively without changing the probability. The only reason for using three parameters is to keep the $\alpha$ and $\beta$ computations in integer arithmetic for faster speed.
\subsection{Local Optimality of \name}
The following theorems
and lemmas are provided to establish the claim of \name's local optimality. 

\begin{lemma}
    When bias $b_\ell = 0$, for any word $a$, if $\mathbf{a}_j$, the $j$'th bit in the binary representation vector $\mathbf{a}$ is updated by \name's probabilistic flipping (keeping the remaining bits the same), the loss function value decreases in the successive iteration.
    \label{lem:flip-pos-a}
\end{lemma}

\begin{lemma}
    When $b_\ell = 0$, for any word $b$, if the bit $\mathbf{b}_j$ alone is updated by \name's probabilistic flipping, the loss decreases.
    \label{lem:flip-pos-b}
\end{lemma}

\begin{lemma}
    When $b_\ell = 0$, given a collection of negative data instances $(a',b')$, if the $j$'th bit in the vector $\mathbf{a}'$ or $\mathbf{b}'$ (and only that bit) is updated by \name's probabilistic flipping, the loss function value decreases or remains the same. 
    \label{lem:flip-neg}
\end{lemma}

\begin{theorem}
If $b_\ell = 0$ and Line 8 of Algorithm \ref{alg:training} is executed sequentially for each index $j$ for each of the entities, \text{\name} reaches a local optimal solution for a 1-Hamming distance neighborhood.
\end{theorem}

Proofs of these statements are provided in Appendix ~\ref{sec:proof of convergence}.
Experimental evidences of convergence to a local optimal value is available from the loss vs iteration curves shown in Figure~\ref{fig:nouns_convergence}.

\section{Experiments and Results\label{sec:experiments}}
To demonstrate the effectiveness of {\name}, we evaluate its performance on two standard tasks for hierarchical 
representation learning: Representation and Transitive Closure (``TC'') Link Prediction. Our representation experimental setup is identical to \cite{boratko_capacity_2021}'s full transitive closure (TC) representation experiment. Our dataset creation and experimental setup for TC link prediction task are identical to \cite{ganea_hyperbolic_2018}. Note that some of the competing works \cite{bhattacharjee_what_2023, nickel_three-way_2011, trouillon_complex_2016, vendrov_order-embeddings_2015, lai_learning_2017} did not report results for the representation and the TC link prediction tasks.
On the other hand, some competitors \cite{ganea_hyperbolic_2018, law_lorentzian_2019} only gave results on the TC link prediction; in particular \cite{law_lorentzian_2019} only reported results on $80\%$ TC link prediction. The latest work on box embedding \cite{boratko_capacity_2021} only reported results on the representation task. Another box embedding work \cite{dasgupta_improving_2020} reported results on the ranking (similar to the reconstruction task) and the $0\%$ TC link prediction task.

{\bf Representation} is an inverse mapping from embedding vectors to the list of positive and negative pairs in the full adjacency matrix.  If $|W| = n$ is the vocabulary size, there are ${n}^2 - n$ edges in the full adjacency matrix.
A high accuracy in representation testifies for the capacity of the learning algorithm: 
it confirms that the learning algorithm obtains embedding vectors to satisfy the representation constraints 
of both the positive and negative pairs. To evaluate representation, 
we train models over the complete data set. We then create a test dataset that includes the entire positive and negative edge set in the full adjacency matrix (this setup is identical to ~\cite{boratko_capacity_2021}). We then validate whether for each positive and negative pair, each model's embedding vectors satisfy respective constraints or not.

{\bf Link Prediction} is predicting edges that the learning algorithm has not seen during training. Since ``is-a'' relation is transitive, transitive-pair edges are redundant: if $(a, b)$ and $(b,c)$ are positive pairs, then $(a, c)$ should also be a positive pair. A basic hierarchical embedding model should be able to deduce this fact automatically. In link prediction task, we validate each of the models' ability to make this deduction.
We include the basic (direct) edges in the train data. These are edges $(a,b)$ with no other word $c$ between them. The remaining non-basic (indirect or transitive) edges are split into validation and test set. For validation and test set, we use 10 times more negative pairs than positive pairs. We add 0\%, 10\%, 25\% and 50\% of the non-basic edges to generate four training sets. Our experimental setup is identical to ~\cite{ganea_hyperbolic_2018}. We then validate for each unseen positive and negative pair in the test set, whether each model's embedding vectors satisfy respective constraints or not. 


The remainder of this Section is as follows. In Section~\ref{subsec:datasets}, we give details of the datasets.
In Section~\ref{subsec:comparison}, we discuss competing methods and our justification for why they were chosen.
In Section~\ref{subsec:explainability}, we show a visual representation of \name's embedding in a small dataset.
In Section ~\ref{sec:training_setup}, we discuss our training setup.
In Sections~\ref{subsec:recon-results} and~\ref{subsec:lp-results}, we compare \name's performance with that of the competing methods.
In Section ~\ref{subsec:resource_consumptions}, we show the space advantage of {\name} embeddings over the competitors.
We report additional details of our experimental setup and further results in Appendix ~\ref{sec:experimental_setup}. In Appendix ~\ref{subsec:edge_distribution}, we discuss the edge distribution for each dataset.
In Appendix ~\ref{subsec:hyperparameter-tuning}, we provide low-level details of our training setup and hyperparameter tuning. In Appendix ~\ref{subsec:ablation} we perform an ablation study to show how different components of the loss function affect \name's performance.
In Appendix ~\ref{subsec:robustness}, we demonstrate the robustness of {\name} by showing how changing hyperparameter values affects \name's performance.
In Appendix ~\ref{sec:convergence}, we show how \name\ loss and $F1$-score converges with the number of iterations for both tasks.

\vspace{-1mm}
\subsection{Datasets\label{subsec:datasets}}
For representation task, we evaluate {\name } on \textbf{6} datasets. We downloaded Music and Medical domain dataset from SemEval-2018 Task 9: Hypernym Discovery. 
We also train on the subtree consisting of all hyponyms of \texttt{animal.n.01} (including itself), which has 4,017 entities and 4,051 direct edges and transitive closure with 29,795 edges.
We collected Lex and random dataset from \cite{shwartz_improving_2016}, which were constructed by extracting hypernymy relations from WordNet \cite{chapelle_wordnet_2012}, DBPedia \cite{auer_dbpedia_2007}, Wiki-data \cite{vrandecic_wikidata_2012} and Yago \cite{suchanek_yago_2007}. The largest dataset we considered is the WordNet noun hierarchy dataset. Its full transitive closure contains 82,115 Nouns and 743,241 hypernymy relations.
This number includes the reflexive relations $w$ is-a $w$ for each word $w$, which we removed for our experiments, leaving 661,127 relational edges.
We generate \emph{negative samples} following a similar method to  \cite{vendrov_order-embeddings_2015, nickel_poincare_2017, ganea_hyperbolic_2018}: we corrupt one of the words in a positive pair, and discard any pairs that happen to be positive pairs.

For prediction, we evaluate {\name} on the same datasets.
For each dataset we show the direct and indirect edge counts in table \ref{tab:data_statistics} and edge distribution percentage in Figure \ref{fig:edge_distribution} in Appendix ~\ref{subsec:edge_distribution}.
We remove the root of WordNet Nouns as its edges are trivial to predict.
The remaining transitive closure of the WordNet Nouns hierarchy consists of 82,114 Nouns and 661,127 hypernymy relations.

\begin{table*}[t!]
  \caption{Representation Experiment results {  } $F1$-score{(\%)}}
  \vspace{-3.5mm}
  \label{tab:representation_exp}
  \centering
  \begin{tabular}{|l|c|c|c|c|c|c|}
    \toprule
    & \textbf{Medical} & \textbf{Music} & \textbf{Animals} & \textbf{Shwartz Lex} & \textbf{Shwartz Random} & \textbf{WordNet Nouns} \\
    Model &  entities = 1.4k  & entities = 1k  &  entities = 4k  & entities = 5.8k  &  entities = 13.2k  & entities = 82k  \\
    & edges = 4.3k & edges = 6.5k  & edges = 29.8k  & edges = 13.5k  & edges = 56.2k  & edges = 743k  \\
    \midrule
    Sim \cite{bhattacharjee_what_2023} & $\mathbf{100}$ & $\mathbf{100}$ & $\mathbf{100}$ & 97.67 & 99.2 & 83.74 \\
    Bilinear \cite{nickel_three-way_2011} & $\mathbf{100}$ & $\mathbf{100}$ & $\mathbf{100}$ & 99.91 & 98.23 & 26.02 \\
    Complex \cite{trouillon_complex_2016} & $\mathbf{100}$ & $\mathbf{100}$ & $\mathbf{100}$ & $\mathbf{100}$ & 99.5 & 94.13 \\
    OE \cite{vendrov_order-embeddings_2015} & 99.88 & 99.94 & 99.76 & 99.9 & 98.82 & 87.27 \\
    POE \cite{lai_learning_2017} & $\mathbf{100}$ & $\mathbf{100}$ & $\mathbf{100}$ & $\mathbf{100}$ & $\mathbf{100}$ & 99.89 \\
    Lorentzian distance \cite{law_lorentzian_2019} & 58.37 & 56.77 & 66.36 & 69.91 & 56.54 & 7.78 \\
    HEC \cite{ganea_hyperbolic_2018} & 97.01 & 96.35 & 95.81 & 94.63 & 94.45 & 57.19 \\
    Gumbel Box \cite{dasgupta_improving_2020} & $\mathbf{100}$ & $\mathbf{100}$ & $\mathbf{100}$ & $\mathbf{100}$ & $\mathbf{100}$ & $\mathbf{100}$ \\
    T Box (/n/d) \cite{boratko_capacity_2021} & $\mathbf{100}$ & $\mathbf{100}$ & $\mathbf{100}$ & $\mathbf{100}$ & $\mathbf{100}$ & 99.96 \\
    \midrule
    BINDER (our model) & $98.26 \pm 0.18$ & $99.16 \pm 0.11$ & $97.98 \pm 1.63$ & $99.12 \pm 0.1$ & $99.3 \pm 0.07$  & $93.34 \pm 0.72$\\
    \bottomrule
\end{tabular}
\vspace{-3.5mm}
\end{table*}

\vspace{-1mm}
\subsection{Competing methods and Metrics Used\label{subsec:comparison}}

We exhaustively compare our model to \textbf{9} existing embedding methods for directed graphs: vectors in Euclidean space (Similarity Vectors \cite{bhattacharjee_what_2023} \& Bilinear Vectors \cite{nickel_three-way_2011}), vectors in Complex space \cite{trouillon_complex_2016}, Hyperbolic embedding using Lorentzian distance learning \cite{law_lorentzian_2019}, the continuous Order Embeddings \cite{vendrov_order-embeddings_2015}, the Probabilistic Order Embeddings \cite{lai_learning_2017}, Hyperbolic Entailment Cones \cite{ganea_hyperbolic_2018}, and the probabilistic box embedding methods, Gumbel Box \cite{dasgupta_improving_2020} and T-Box \cite{boratko_capacity_2021}.
All these methods are intended to produce embedding for entities having hierarchical organization.
Among the above methods, our model is most similar to Order Embedding~\cite{vendrov_order-embeddings_2015}, as our model is simply the restriction of theirs from $(\mathbb{R}^+)^d$ to $\{0, 1\}^d$.
So, this model is a natural competitor.
{\name} is also similar to Probabilistic Order Embedding, Box Embeddings and Hyperbolic Entailment Cones \cite{ganea_hyperbolic_2018}, as they all are transitive order embeddings. 
We also compare to Hyperbolic Embeddings \cite{nickel_poincare_2017, nickel_learning_2018, law_lorentzian_2019}, which uses a distance ranking function to embed the undirected version of the hierarchy graph. We consider Lorentzian distance learning over Poncar\'e Embeddings \cite{law_lorentzian_2019} as it is the latest among hyperbolic methods. In subsequent discussion and result
tables, we will denote Similarity Vector model as Sim, Bilinear vector model as Bilinear, Complex vector model as Complex, Lorentzian Distance Learning Hyperbolic model as Lorentzian Distance, Order Embeddings model as OE, Probabilistic Order Embeddings as POE, and Hyperbolic Entailment Cones as HEC. 
We report $F1$-score for our experimental results, as it is a preferred metric for imbalanced data.

\begin{figure}[b]
    \centering
    \vspace{-2mm}
    \includegraphics[width=0.40\textwidth]{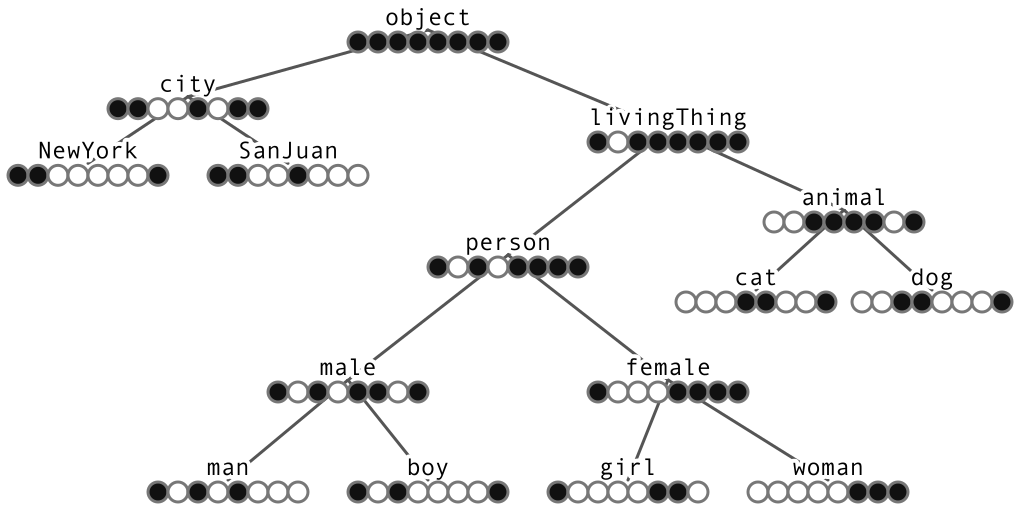}

    \vspace{-3mm}
    \caption{Visual representation of toy dataset results. White circles represent 1 and black circles 0.}
    \label{fig:toyvisual}
    \vspace{-5mm}
\end{figure}

\vspace{-1mm}
\subsection{Case Study\label{subsec:explainability}}

We present a case-study experiment, which will provide the reader a sketch of \name's embedding results. For this
we run our model on a toy dataset from \cite{vendrov_order-embeddings_2015} to generate 8-bit
embedding vectors for 15 entities; \name\ achieves perfect score in this task satisfying all constraints.
Figure \ref{fig:toyvisual} shows the embedding vectors in a tree structure. By observing the embedding vectors of  \texttt{boy}, \texttt{person}, and \texttt{city}, anyone can determine that \texttt{boy} is-a \texttt{person} 
but is-not-a \texttt{city}. We can grasp \name's embedding by thinking that 
each bit corresponds to some ``attribute'' of the object, where all attributes are passed down to hyponyms.
This can help to build intuitively interpretable embeddings for the concepts in a dataset.
For instance, the second bit can be an attribute which could be named as \textit{has-life}, since it is the 
only 1 bit in the \texttt{livingThing} embedding. Sometimes, however, bits can be used with different meanings 
in different words: the right-most bit is 1 on \texttt{man}, \texttt{girl}, and \texttt{SanJuan} embeddings. 
This is partly because the toy dataset is very sparse, with only two sub-components of \texttt{city} compared to ten sub-components of \texttt{livingThing}, and \texttt{adult} and \texttt{child} were not included in the dataset, as they are in WordNet. If we consider the attributes as intent, and objects as extent, \name's binary embedding 
is a manifestation of intent-extent relationship in formal concept analysis.

\begin{table*}[t] 
  \caption{Transitive Closure Link Prediction Experiment results {  } $F1$-score{(\%)}}
  \label{tab:transitive_closure_exp}
  \vspace{-3mm}
  \centering
  \begin{tabular}{|l|c|c|c|c|c|c|}
    \toprule
    Model & \textbf{Medical} & \textbf{Music} & \textbf{Animals} & \textbf{Shwartz Lex} & \textbf{Shwartz Random} & \textbf{WordNet Nouns} \\
    \midrule
    \multicolumn{7}{c}{\textbf{Transitive Closure 0\%}}\\
    \midrule
    Sim & 21.23 & 25.69 & 18.64 & 18.69 & 22.96 & 23.54 \\
    Bilinear & 21.92 & 26.24 & 26.05 & 28.05 & 16.68 & 25.71 \\
    Complex & 23.04 & 29.27 & 23.35 & 17.46 & 16.77 & 24.13 \\
    OE & 55.9 & 46.94 & 33.99 & 32.29 & 30.57 & 28.52 \\
    POE & 81.18 & 48.3 & 31.81 & 30.66 & 28.61 & 27.26 \\
    Lorentzian distance & 62.86 & 33.33 & 33.53 & 40.41 & 34.01 & 28.26 \\
    HEC & 53.4 & 35.88 & 32.33 & 36.38 & 32.92 & 29.03 \\
    Gumbel Box & 88.51 & 64.44 & 50.65 & 46.56 & 29.74 & 28.3 \\
    T Box (/n/d) & 71.05 & 46.77 & 33.49 & 30.45 & 27.53 & 27.84 \\
    \midrule
    BINDER & $\mathbf{9 8 . 6 1} \pm \mathbf{1 . 3 2}$ & $\mathbf{9 2 . 1 3 \pm \mathbf { 2 . 2 7 }}$ & $\mathbf{9 8 . 9 9} \pm \mathbf{0 . 2 4}$ & $\mathbf{9 9 . 8 7} \pm \mathbf{0 . 0 9}$ & $\mathbf{9 9 . 5 2} \pm \mathbf{0 . 2 1}$ & $\mathbf{9 8 . 6 4} \pm \mathbf{0 . 0 9}$ \\
    \midrule
    \multicolumn{7}{c}{\textbf{Transitive Closure 10\%}}\\
    \midrule
    Sim & 55.42 & 40.31 & 38.15 & 43.88 & 31.68 & 27.89 \\
    Bilinear & 40.48 & 37.4 & 31.74 & 32.36 & 16.69 & 16.67 \\
    Complex & 38.42 & 42.32 & 19.36 & 17.95 & 17.06 & 16.93 \\
    OE & 52.63 & 50 & 33.37 & 33.7 & 35.75 & 28.22 \\
    POE & 93.83 & 77.19 & 48.16 & 32.15 & 34.37 & 33.35 \\
    Lorentzian distance & 68.02 & 47.67 & 42 & 69.53 & 62.39 & 54.75 \\
    HEC & 66.67 & 45.97 & 54.37 & 78.12 & 70.07 & 65.11 \\
    Gumbel Box & 93.67 & 83.27 & 65.26 & 58.58 & 52.54 & 32.33 \\
    T Box (/n/d) & 76.32 & 60.35 & 47.24 & 37.37 & 36.02 & 30.33 \\
    \midrule
    BINDER & $\mathbf{99.41 \pm 0.42}$ & $\mathbf{93.9 \pm 1.36}$ & $\mathbf{98 \pm 0.9}$ & $\mathbf{99.95 \pm 0.07}$ & $\mathbf{99.33 \pm 0.85}$ & $\mathbf{98.49 \pm 0.21}$ \\
    \midrule
    \multicolumn{7}{c}{\textbf{Transitive Closure 25\%}}\\
    \midrule
    Sim & 64 & 55.26 & 50.23 & 38.31 & 38.19 & 28.62 \\
    Bilinear & 51.81 & 50 & 44.33 & 18.4 & 16.68 & 16.69 \\
    Complex & 58.29 & 57 & 31.27 & 18.02 & 17.6 & 16.78 \\
    OE & 52.34 & 48.11 & 33.17 & 35.85 & 39.16 & 28.18 \\
    POE & 82.02 & 53.03 & 31.87 & 38.57 & 39.18 & 28.05 \\
    Lorentzian distance & 69.31 & 46.2 & 61.61 & 55.1 & 77.27 & 59.51 \\
    HEC & 72.63 & 56.58 & 53.6 & 81 & 68.77 & 56.68 \\
    Gumbel Box & 98.18 & 89.05 & 80.98 & 67.31 & 67.38 & 31.42 \\
    T Box (/n/d) & 81.88 & 58.12 & 52.95 & 34.6 & 45.28 & 29.87 \\
    \midrule
    BINDER & $\mathbf{99.53 \pm 0.76}$ & $\mathbf{92.78 \pm 1.58}$ & $\mathbf{99.05 \pm 0.8}$ & $\mathbf{99.97 \pm 0.06}$ & $\mathbf{98.78 \pm 0.51}$ & $\mathbf{99.04 \pm 0.14}$ \\
    \midrule
    \multicolumn{7}{c}{\textbf{Transitive Closure 50\%}}\\
    \midrule
    Sim & 82.35 & 72.22 & 67.04 & 49.44 & 38.53 & 28.91 \\
    Bilinear & 69.01 & 66.67 & 17.90 & 17.14 & 16.68 & 16.67 \\
    Complex & 50 & 36.05 & 31.47 & 18.24 & 17.35 & 16.73 \\
    OE & 55.51 & 50.23 & 32.37 & 37.48 & 39.97 & 28.16 \\
    POE & 61.22 & 60.52 & 31.96 & 36.2 & 43.27 & 28.09 \\
    Lorentzian distance & 75.9 & 46.49 & 65.52 & 45.82 & 85.03 & 61.16 \\
    HEC & 79.76 & 59.11 & 51.96 & 79.15 & 66.44 & 57.79 \\
    Gumbel Box & 99.41 & $\mathbf{97.28}$ & 90.72 & 62.99 & 48.84 & 28.81 \\
    T Box (/n/d) & 81.76 & 72.22 & 64.26 & 38.45 & 49.14 & 34.81 \\
    \midrule
    BINDER & $\mathbf{99.42 \pm 1.0}$ & $93.8 \pm 0.87$ & $\mathbf{99.9 \pm 0.1}$ & $\mathbf{99.92 \pm 0.07}$ & $\mathbf{99.85 \pm 0.14}$ & $\mathbf{99.56 \pm 0.03}$\\
    \bottomrule
\end{tabular}
\end{table*}

\vspace{-1mm}
\subsection{Training Setup \label{sec:training_setup}}

For \name, we learn a $d$-bit array for each concept in the hierarchy.
For all tasks, we train \name\ for 10000 epochs, where each epoch considers the full batch for training.
We tune our hyper-parameters: dimension $d$, positive and negative sample weights $\alpha$, $\beta$, negative sample multiplier $n^-$, and learning rate $r_\ell$ and bias $b_\ell$ manually by running separate experiments for each dataset.
We find that the optimal learning rate $r_\ell$ and learning bias $b_\ell$ are 0.008 and 0.01 respectively for all datasets and tasks.
Although ablation study results in Section ~\ref{subsec:ablation} show that the learning bias $b_\ell$ has no effect on the results reported, we still keep it at $b_\ell = 0.01$, so bits whose gradient was exactly neutral have a 1\% chance of flipping.
We fix $\beta = 10$ and tune $\alpha$; we always find that $\alpha \leq \beta$ gives far too many false negatives.
We report best hyperparameter configurations for all tasks in Table ~\ref{tab:best_hyperparameters} in Appendix ~\ref{subsec:hyperparameter-tuning}.
For the representation task we need more bits to increase the capacity of our model to better reconstruct training edges.
We also extensively hyperparameter tuned our competitors, and we provide details of hyperparameter ranges that we considered in Table ~\ref{tab:hyperparameter_range} in Appendix ~\ref{subsec:hyperparameter-tuning}.

\vspace{-1mm}
\subsection{Representation Task Results\label{subsec:recon-results}}

This task is easier; all the competing methods perform better on this task than on link prediction.
This is because representation is similar to a fitting task (like training accuracy in ML) and with higher dimensions, models have more freedom (higher capacity) to satisfy the order constraints in the training data.
The results are shown in Table~\ref{tab:representation_exp}, 
in which the datasets are arranged in increasing size from left to right.
For each dataset and each model, we report the $F1$-scores.
Since \name\ is a randomized algorithm, we show mean and standard deviation of the results obtained in 5 runs. 

\name, and all the competing methods except hyperbolic embedding (HEC, Lorentzian), perform very well on the smaller datasets (entity count around 
5K)---some of the competing methods reach 100\% $F1$-score, while \name\ 
achieves 98\% to 99.5\% $F1$-score. However, on the largest 
dataset, WordNet (82K entities), their $F1$ score drops drastically (for Bilinear the drop is from 100\% to 26\%), except for T-Box and Gumbel Box, which retain their performance. \name's performance on WordNet is about 93.3\%. Given
that there are a finite number of bit vectors, \name\ has a limited fitting
capacity; considering this, \name's results are excellent. {\name} 
reaches very good embeddings in a relatively small number of iterations (see bottom left graph in Figure ~\ref{fig:nouns_convergence} in Appendix ~\ref{sec:convergence}), but near a local minima,
the flipping probability becomes small (due to low gradient), which causes the
method not to achieve perfect fit (100\% $F1$-score) like some of the other methods. Since for the representation task we include all positive pairs in training, a perfect $F1$-score indicates that a model is good at fitting (possibly overfitting) training data, which may work against the model to achieve generalization on unseen data. On the other hand, the apparent low capacity of \name\ is a blessing in disguise, which we show in the next section where we compare \name's performance with the competitors on link prediction of unseen transitive edges.


\vspace{-1mm}
\subsection{Link Prediction Results\label{subsec:lp-results}}
Link prediction task is similar to test accuracy in a typical machine learning 
task. It is more challenging and practically useful than representation, as prediction is performed on unseen edges. The results of this task for all datasets and all
methods are shown in Table \ref{tab:transitive_closure_exp}.
As we can see from the results, {\name} handsomely wins over all competitors in all transitive closure (TC) link prediction tasks (except the Music dataset in TC 50\% task). In fact in all datasets, \name's $F1$-score
is around 95\% or more, whereas the competing methods struggle to reach even 50\% $F1$-score.
Note that, in a tree-like dataset, there are generally much more transitive edges than direct edges.
As we add more transitive edges to the training dataset, the competitors' results improve, which shows that their methods rely on these edges for link prediction, whereas {\name} does not (\name's $F1$-scores are between 98\% to 99.5\% for 0\%, 10\%, 25\%, and 50\% TC edges). 
Other models' performances suffer significantly as the dataset size increases, whereas {\name} maintains its performance, and thus is more scalable.
For 0\% TC link prediction on the largest dataset, WordNet (with 743,241 edges), {\name}'s $F1$-score (98.5\%) surpassed that of the best
competing models (29\%) by about 70\%. From our investigation, \name's superiority comes from the nature of its constraints, which enable it to predict negative pairs ($a$ is-not-a $b$) by creating only one $(0,1)$ bit pair.

\vspace{-1mm}
\subsection{Space advantage of \name \label{subsec:resource_consumptions} }
One of the main advantages of using binary vectors is their space efficiency compared to floating-point vectors. The final representation of the concepts using \name\ are bit vectors, so the storage required for $n$ words and $d$ dimensions is $\frac{dn}{8}$ bytes. This is much smaller than OE and HEC's $4dn$ bytes, Gumbel Box's $8dn$ bytes, and TBox's $16dn$ bytes if 32-bit floating point values are used.
If we consider the WordNet Nouns dataset and $d=100$, the vector, Order Embedding and Hyperbolic methods takes at least 34.2 MB for their embedding (including the 1.33 MB required for the WordNet dataset itself), while box embedding methods take up to 67 MB in storage space.
In contrast, {\name} takes only 2.36 MB for its embedding, which is significantly smaller than its floating-point vector competitors. We did use the floating-point \emph{tanh} function during training for probability computations, but this was mostly for convenience; we expect that discrete probability functions, and reductions in memory usage during training, can be derived and used alongside {\name} in a rigorous way in future works.

\vspace{-0.5mm}
\section{Related Works}

We have discussed about the works that perform hierarchical embedding in the Introduction section and compared our results with those in the Results section. Binarized embedding was previously explored in Knowledge Graph space.
Authors in \cite{wang_learning_2019} proposed learning to hash using compact binary codes for KG completion and query tasks.
They show that for large-scale similarity computation, binary codes are computationally efficient over continuous dense vectors.
To further address the challenge of large storage and computational cost of continuous graph embedding methods, \cite{li_discrete_2021} proposed a discrete knowledge graph embedding (DKGE) method.
The closest to our work is \cite{hayashi_greedy_2020} which proposed a discrete optimization method for training a binarized knowledge graph embedding model. 
But unlike \name, all of these approaches either use continuous real-valued vectors \cite{wang_learning_2019, li_discrete_2021} or a combination of binary and continuous real-valued vectors \cite{hayashi_greedy_2020} for optimization process.
Binary code is also used with Box Embedding method to model directed graphs that contain cycles \cite{zhang_modeling_2022}.
For data compression while preserving the Euclidean distance between a pair of entities, \cite{zhang_faster_2021} used $k$-dimensional binary vector space $[-1, +1]^k$.

To the best of our knowledge, we are the first to propose hierarchical concept representation in binary space. Our use of intent-extent philosophy of formal concept analysis (FCA) to derive hierarchical concept representation is also novel.
Previously, \cite{rudolph_using_2007} used FCA to encode a formal context's closure operators into NN, and \cite{durrschnabel_fca2vec_2019} introduced fca2vec similar to node2vec \cite{grover_node2vec_2016}, which embeds existing FCA systems into real-valued vector spaces using NN method. None of these works are our competitor, as their embedding objective is to embed in 2-dimensional real space for visualization. 


\vspace{-0.5mm}
\section{Future Works and Conclusion}

{\name} is the first work to use binary vectors for embedding hierarchical concepts, so there are numerous scopes for building on top of this work.
First, we wish to explore more efficient combinatorial optimization algorithms to replace \name's learning algorithm, by using various well-known CSP heuristics and exploring discrete randomization to replace the \emph{tanh} function.
\name’s Loss function can also be extended with a node similarity expression based on Hamming distance between embedding vectors, which we plan to do next.
In terms of limitations, \name\ is a transductive model, i.e. a concept must appear in the training data for its embedding to be learnt, but this limitation is also shared by existing hierarchical embedding models.
However, \name\ can generate embeddings for derived concepts by using logical functions over existing concepts more readily than other methods.
Another future work can be to make \name\ inductive over unseen concepts by including knowledge about the broadness or narrowness of a concept from a large distributed language model, such as BERT, RoBERTa, or GLoVe.

To conclude, in this work, we propose \name, a novel approach for order embedding using binary vectors. {\name} is ideal for finding 
representation of concepts exhibiting hypernym-hyponym relationship. Also, \name's binary vector based embedding is extensive 
as it allows obtaining embedding of other concepts which are logical derivatives of existing concepts. 
Experiments on 6 benchmark datasets
show that \name\ is superior than the existing state-of-the-art order embedding methodologies in the transitive closure link prediction tasks.


\vspace{-0.5mm}
\section{Acknowledgements}
This material is based upon work supported by the National Science Foundation (NSF), USA under Grant No. IIS-1909916.

\bibliographystyle{ACM-Reference-Format}
\bibliography{references.bib}


\begin{thebibliography}{52}


\ifx \showCODEN    \undefined \def \showCODEN     #1{\unskip}     \fi
\ifx \showDOI      \undefined \def \showDOI       #1{#1}\fi
\ifx \showISBNx    \undefined \def \showISBNx     #1{\unskip}     \fi
\ifx \showISBNxiii \undefined \def \showISBNxiii  #1{\unskip}     \fi
\ifx \showISSN     \undefined \def \showISSN      #1{\unskip}     \fi
\ifx \showLCCN     \undefined \def \showLCCN      #1{\unskip}     \fi
\ifx \shownote     \undefined \def \shownote      #1{#1}          \fi
\ifx \showarticletitle \undefined \def \showarticletitle #1{#1}   \fi
\ifx \showURL      \undefined \def \showURL       {\relax}        \fi
\providecommand\bibfield[2]{#2}
\providecommand\bibinfo[2]{#2}
\providecommand\natexlab[1]{#1}
\providecommand\showeprint[2][]{arXiv:#2}

\bibitem[Auer et~al\mbox{.}(2007)]%
        {auer_dbpedia_2007}
\bibfield{author}{\bibinfo{person}{Sören Auer}, \bibinfo{person}{Christian
  Bizer}, \bibinfo{person}{Georgi Kobilarov}, \bibinfo{person}{Jens Lehmann},
  \bibinfo{person}{Richard Cyganiak}, {and} \bibinfo{person}{Zachary Ives}.}
  \bibinfo{year}{2007}\natexlab{}.
\newblock \showarticletitle{{DBpedia}: {A} {Nucleus} for a {Web} of {Open}
  {Data}}. In \bibinfo{booktitle}{\emph{The {Semantic} {Web}}}
  \emph{(\bibinfo{series}{Lecture {Notes} in {Computer} {Science}})},
  \bibfield{editor}{\bibinfo{person}{Karl Aberer}, \bibinfo{person}{Key-Sun
  Choi}, \bibinfo{person}{Natasha Noy}, \bibinfo{person}{Dean Allemang},
  \bibinfo{person}{Kyung-Il Lee}, \bibinfo{person}{Lyndon Nixon},
  \bibinfo{person}{Jennifer Golbeck}, \bibinfo{person}{Peter Mika},
  \bibinfo{person}{Diana Maynard}, \bibinfo{person}{Riichiro Mizoguchi},
  \bibinfo{person}{Guus Schreiber}, {and} \bibinfo{person}{Philippe
  Cudré-Mauroux}} (Eds.). \bibinfo{publisher}{Springer},
  \bibinfo{address}{Berlin, Heidelberg}, \bibinfo{pages}{722--735}.
\newblock
\showISBNx{9783540762980}
\urldef\tempurl%
\url{https://doi.org/10.1007/978-3-540-76298-0_52}
\showDOI{\tempurl}


\bibitem[Beck et~al\mbox{.}(2011)]%
        {beck_combining_2011}
\bibfield{author}{\bibinfo{person}{J.~Christopher Beck}, \bibinfo{person}{T.~K.
  Feng}, {and} \bibinfo{person}{Jean-Paul Watson}.}
  \bibinfo{year}{2011}\natexlab{}.
\newblock \showarticletitle{Combining {Constraint} {Programming} and {Local}
  {Search} for {Job}-{Shop} {Scheduling}}.
\newblock \bibinfo{journal}{\emph{INFORMS Journal on Computing}}
  \bibinfo{volume}{23}, \bibinfo{number}{1} (\bibinfo{date}{Feb.}
  \bibinfo{year}{2011}), \bibinfo{pages}{1--14}.
\newblock
\showISSN{1091-9856, 1526-5528}
\urldef\tempurl%
\url{https://doi.org/10.1287/ijoc.1100.0388}
\showDOI{\tempurl}


\bibitem[Bhattacharjee and Dasgupta(2023)]%
        {bhattacharjee_what_2023}
\bibfield{author}{\bibinfo{person}{Robi Bhattacharjee} {and}
  \bibinfo{person}{Sanjoy Dasgupta}.} \bibinfo{year}{2023}\natexlab{}.
\newblock \bibinfo{title}{What relations are reliably embeddable in {Euclidean}
  space?}
\newblock
\newblock
\urldef\tempurl%
\url{https://doi.org/10.48550/arXiv.1903.05347}
\showDOI{\tempurl}
\newblock
\shownote{arXiv:1903.05347 [cs, stat]}.


\bibitem[Boratko et~al\mbox{.}(2021)]%
        {boratko_capacity_2021}
\bibfield{author}{\bibinfo{person}{Michael Boratko}, \bibinfo{person}{Dongxu
  Zhang}, \bibinfo{person}{Nicholas Monath}, \bibinfo{person}{Luke Vilnis},
  \bibinfo{person}{Kenneth~L Clarkson}, {and} \bibinfo{person}{Andrew
  McCallum}.} \bibinfo{year}{2021}\natexlab{}.
\newblock \showarticletitle{Capacity and {Bias} of {Learned} {Geometric}
  {Embeddings} for {Directed} {Graphs}}. In \bibinfo{booktitle}{\emph{Advances
  in {Neural} {Information} {Processing} {Systems}}},
  Vol.~\bibinfo{volume}{34}. \bibinfo{publisher}{Curran Associates, Inc.},
  \bibinfo{pages}{16423--16436}.
\newblock
\urldef\tempurl%
\url{https://proceedings.neurips.cc/paper/2021/hash/88d25099b103efd638163ecb40a55589-Abstract.html}
\showURL{%
\tempurl}


\bibitem[Bordes et~al\mbox{.}(2013)]%
        {bordes_translating_2013}
\bibfield{author}{\bibinfo{person}{Antoine Bordes}, \bibinfo{person}{Nicolas
  Usunier}, \bibinfo{person}{Alberto Garcia-Durán}, \bibinfo{person}{Jason
  Weston}, {and} \bibinfo{person}{Oksana Yakhnenko}.}
  \bibinfo{year}{2013}\natexlab{}.
\newblock \showarticletitle{Translating embeddings for modeling
  multi-relational data}. In \bibinfo{booktitle}{\emph{Proceedings of the 26th
  {International} {Conference} on {Neural} {Information} {Processing} {Systems}
  - {Volume} 2}} \emph{(\bibinfo{series}{{NIPS}'13})}.
  \bibinfo{publisher}{Curran Associates Inc.}, \bibinfo{address}{Red Hook, NY,
  USA}, \bibinfo{pages}{2787--2795}.
\newblock


\bibitem[Cormen et~al\mbox{.}(2022)]%
        {cormen_introduction_2022}
\bibfield{author}{\bibinfo{person}{Thomas~H. Cormen},
  \bibinfo{person}{Charles~E. Leiserson}, \bibinfo{person}{Ronald~L. Rivest},
  {and} \bibinfo{person}{Clifford Stein}.} \bibinfo{year}{2022}\natexlab{}.
\newblock \bibinfo{booktitle}{\emph{Introduction to {Algorithms}, fourth
  edition}}.
\newblock \bibinfo{publisher}{MIT Press}.
\newblock
\showISBNx{9780262367509}
\newblock
\shownote{Google-Books-ID: RSMuEAAAQBAJ}.


\bibitem[Dai et~al\mbox{.}(2016)]%
        {dai_cfo_2016}
\bibfield{author}{\bibinfo{person}{Zihang Dai}, \bibinfo{person}{Lei Li}, {and}
  \bibinfo{person}{Wei Xu}.} \bibinfo{year}{2016}\natexlab{}.
\newblock \bibinfo{title}{{CFO}: {Conditional} {Focused} {Neural} {Question}
  {Answering} with {Large}-scale {Knowledge} {Bases}}.
\newblock
\newblock
\urldef\tempurl%
\url{https://doi.org/10.48550/arXiv.1606.01994}
\showDOI{\tempurl}
\newblock
\shownote{arXiv:1606.01994 [cs]}.


\bibitem[Dasgupta et~al\mbox{.}(2020)]%
        {dasgupta_improving_2020}
\bibfield{author}{\bibinfo{person}{Shib~Sankar Dasgupta},
  \bibinfo{person}{Michael Boratko}, \bibinfo{person}{Dongxu Zhang},
  \bibinfo{person}{Luke Vilnis}, \bibinfo{person}{Xiang~Lorraine Li}, {and}
  \bibinfo{person}{Andrew McCallum}.} \bibinfo{year}{2020}\natexlab{}.
\newblock \showarticletitle{Improving {Local} {Identifiability} in
  {Probabilistic} {Box} {Embeddings}}.
\newblock \bibinfo{journal}{\emph{NeurIPS 2020 (Virtual)}}
  (\bibinfo{date}{Oct.} \bibinfo{year}{2020}).
\newblock
\urldef\tempurl%
\url{https://doi.org/10.48550/arXiv.2010.04831}
\showDOI{\tempurl}


\bibitem[Dasgupta et~al\mbox{.}(2021)]%
        {dasgupta_box--box_2021}
\bibfield{author}{\bibinfo{person}{Shib~Sankar Dasgupta},
  \bibinfo{person}{Xiang~Lorraine Li}, \bibinfo{person}{Michael Boratko},
  \bibinfo{person}{Dongxu Zhang}, {and} \bibinfo{person}{Andrew McCallum}.}
  \bibinfo{year}{2021}\natexlab{}.
\newblock \showarticletitle{Box-{To}-{Box} {Transformations} for {Modeling}
  {Joint} {Hierarchies}}. In \bibinfo{booktitle}{\emph{Proceedings of the 6th
  {Workshop} on {Representation} {Learning} for {NLP} ({RepL4NLP}-2021)}}.
  \bibinfo{publisher}{Association for Computational Linguistics},
  \bibinfo{address}{Online}, \bibinfo{pages}{277--288}.
\newblock
\urldef\tempurl%
\url{https://doi.org/10.18653/v1/2021.repl4nlp-1.28}
\showDOI{\tempurl}


\bibitem[Dürrschnabel et~al\mbox{.}(2019)]%
        {durrschnabel_fca2vec_2019}
\bibfield{author}{\bibinfo{person}{Dominik Dürrschnabel}, \bibinfo{person}{Tom
  Hanika}, {and} \bibinfo{person}{Maximilian Stubbemann}.}
  \bibinfo{year}{2019}\natexlab{}.
\newblock \bibinfo{title}{{FCA2VEC}: {Embedding} {Techniques} for {Formal}
  {Concept} {Analysis}}.
\newblock
\newblock
\urldef\tempurl%
\url{https://doi.org/10.48550/arXiv.1911.11496}
\showDOI{\tempurl}
\newblock
\shownote{arXiv:1911.11496 [cs, stat]}.


\bibitem[Fellbaum(2012)]%
        {chapelle_wordnet_2012}
\bibfield{author}{\bibinfo{person}{Christiane Fellbaum}.}
  \bibinfo{year}{2012}\natexlab{}.
\newblock \showarticletitle{{WordNet}}.
\newblock In \bibinfo{booktitle}{\emph{The {Encyclopedia} of {Applied}
  {Linguistics}}}, \bibfield{editor}{\bibinfo{person}{Carol Chapelle}} (Ed.).
  \bibinfo{publisher}{John Wiley \& Sons, Inc.}, \bibinfo{address}{Hoboken, NJ,
  USA}, \bibinfo{pages}{wbeal1285}.
\newblock
\showISBNx{9781405198431}
\urldef\tempurl%
\url{https://doi.org/10.1002/9781405198431.wbeal1285}
\showDOI{\tempurl}


\bibitem[Galárraga et~al\mbox{.}(2013)]%
        {galarraga_amie_2013}
\bibfield{author}{\bibinfo{person}{Luis~Antonio Galárraga},
  \bibinfo{person}{Christina Teflioudi}, \bibinfo{person}{Katja Hose}, {and}
  \bibinfo{person}{Fabian Suchanek}.} \bibinfo{year}{2013}\natexlab{}.
\newblock \showarticletitle{{AMIE}: association rule mining under incomplete
  evidence in ontological knowledge bases}. In
  \bibinfo{booktitle}{\emph{Proceedings of the 22nd international conference on
  {World} {Wide} {Web}}} \emph{(\bibinfo{series}{{WWW} '13})}.
  \bibinfo{publisher}{Association for Computing Machinery},
  \bibinfo{address}{New York, NY, USA}, \bibinfo{pages}{413--422}.
\newblock
\showISBNx{9781450320351}
\urldef\tempurl%
\url{https://doi.org/10.1145/2488388.2488425}
\showDOI{\tempurl}


\bibitem[Ganea et~al\mbox{.}(2018)]%
        {ganea_hyperbolic_2018}
\bibfield{author}{\bibinfo{person}{Octavian-Eugen Ganea}, \bibinfo{person}{Gary
  Bécigneul}, {and} \bibinfo{person}{Thomas Hofmann}.}
  \bibinfo{year}{2018}\natexlab{}.
\newblock \showarticletitle{Hyperbolic {Entailment} {Cones} for {Learning}
  {Hierarchical} {Embeddings}}.
\newblock \bibinfo{journal}{\emph{International Conference on Machine Learning
  (ICML) Stockholm, Sweden.}} (\bibinfo{date}{April} \bibinfo{year}{2018}).
\newblock
\urldef\tempurl%
\url{https://doi.org/10.48550/arXiv.1804.01882}
\showDOI{\tempurl}


\bibitem[Ganter and Wille(1999)]%
        {ganter_formal_1999}
\bibfield{author}{\bibinfo{person}{Bernhard Ganter} {and}
  \bibinfo{person}{Rudolf Wille}.} \bibinfo{year}{1999}\natexlab{}.
\newblock \bibinfo{booktitle}{\emph{Formal {Concept} {Analysis}}}.
\newblock \bibinfo{publisher}{Springer}, \bibinfo{address}{Berlin, Heidelberg}.
\newblock
\showISBNx{9783540627715 9783642598302}
\urldef\tempurl%
\url{https://doi.org/10.1007/978-3-642-59830-2}
\showDOI{\tempurl}


\bibitem[Grover and Leskovec(2016)]%
        {grover_node2vec_2016}
\bibfield{author}{\bibinfo{person}{Aditya Grover} {and} \bibinfo{person}{Jure
  Leskovec}.} \bibinfo{year}{2016}\natexlab{}.
\newblock \showarticletitle{node2vec: {Scalable} {Feature} {Learning} for
  {Networks}}.
\newblock  (\bibinfo{date}{July} \bibinfo{year}{2016}).
\newblock
\urldef\tempurl%
\url{https://doi.org/10.48550/arXiv.1607.00653}
\showDOI{\tempurl}


\bibitem[Guo et~al\mbox{.}(2019)]%
        {guo_learning_2019}
\bibfield{author}{\bibinfo{person}{Lingbing Guo}, \bibinfo{person}{Zequn Sun},
  {and} \bibinfo{person}{Wei Hu}.} \bibinfo{year}{2019}\natexlab{}.
\newblock \bibinfo{title}{Learning to {Exploit} {Long}-term {Relational}
  {Dependencies} in {Knowledge} {Graphs}}.
\newblock
\newblock
\urldef\tempurl%
\url{https://doi.org/10.48550/arXiv.1905.04914}
\showDOI{\tempurl}
\newblock
\shownote{arXiv:1905.04914 [cs]}.


\bibitem[Hasan and Zaki(2011)]%
        {hasan_survey_2011}
\bibfield{author}{\bibinfo{person}{Mohammad~Al Hasan} {and}
  \bibinfo{person}{Mohammed~J. Zaki}.} \bibinfo{year}{2011}\natexlab{}.
\newblock \showarticletitle{A {Survey} of {Link} {Prediction} in {Social}
  {Networks}}.
\newblock In \bibinfo{booktitle}{\emph{Social {Network} {Data} {Analytics}}},
  \bibfield{editor}{\bibinfo{person}{Charu~C. Aggarwal}} (Ed.).
  \bibinfo{publisher}{Springer US}, \bibinfo{address}{Boston, MA},
  \bibinfo{pages}{243--275}.
\newblock
\showISBNx{9781441984623}
\urldef\tempurl%
\url{https://doi.org/10.1007/978-1-4419-8462-3_9}
\showDOI{\tempurl}


\bibitem[Hayashi et~al\mbox{.}(2020)]%
        {hayashi_greedy_2020}
\bibfield{author}{\bibinfo{person}{Katsuhiko Hayashi}, \bibinfo{person}{Koki
  Kishimoto}, {and} \bibinfo{person}{Masashi Shimbo}.}
  \bibinfo{year}{2020}\natexlab{}.
\newblock \showarticletitle{A {Greedy} {Bit}-flip {Training} {Algorithm} for
  {Binarized} {Knowledge} {Graph} {Embeddings}}. In
  \bibinfo{booktitle}{\emph{Findings of the {Association} for {Computational}
  {Linguistics}: {EMNLP} 2020}}, \bibfield{editor}{\bibinfo{person}{Trevor
  Cohn}, \bibinfo{person}{Yulan He}, {and} \bibinfo{person}{Yang Liu}} (Eds.).
  \bibinfo{publisher}{Association for Computational Linguistics},
  \bibinfo{address}{Online}, \bibinfo{pages}{109--114}.
\newblock
\urldef\tempurl%
\url{https://doi.org/10.18653/v1/2020.findings-emnlp.10}
\showDOI{\tempurl}


\bibitem[He et~al\mbox{.}(2015)]%
        {he_learning_2015}
\bibfield{author}{\bibinfo{person}{Shizhu He}, \bibinfo{person}{Kang Liu},
  \bibinfo{person}{Guoliang Ji}, {and} \bibinfo{person}{Jun Zhao}.}
  \bibinfo{year}{2015}\natexlab{}.
\newblock \showarticletitle{Learning to {Represent} {Knowledge} {Graphs} with
  {Gaussian} {Embedding}}. In \bibinfo{booktitle}{\emph{Proceedings of the 24th
  {ACM} {International} on {Conference} on {Information} and {Knowledge}
  {Management}}} \emph{(\bibinfo{series}{{CIKM} '15})}.
  \bibinfo{publisher}{Association for Computing Machinery},
  \bibinfo{address}{New York, NY, USA}, \bibinfo{pages}{623--632}.
\newblock
\showISBNx{9781450337946}
\urldef\tempurl%
\url{https://doi.org/10.1145/2806416.2806502}
\showDOI{\tempurl}


\bibitem[Karpathy and Fei-Fei(2015)]%
        {karpathy_deep_2015}
\bibfield{author}{\bibinfo{person}{Andrej Karpathy} {and} \bibinfo{person}{Li
  Fei-Fei}.} \bibinfo{year}{2015}\natexlab{}.
\newblock \bibinfo{title}{Deep {Visual}-{Semantic} {Alignments} for
  {Generating} {Image} {Descriptions}}.
\newblock
\newblock
\urldef\tempurl%
\url{https://doi.org/10.48550/arXiv.1412.2306}
\showDOI{\tempurl}
\newblock
\shownote{arXiv:1412.2306 [cs]}.


\bibitem[Lai and Hockenmaier(2017)]%
        {lai_learning_2017}
\bibfield{author}{\bibinfo{person}{Alice Lai} {and} \bibinfo{person}{Julia
  Hockenmaier}.} \bibinfo{year}{2017}\natexlab{}.
\newblock \showarticletitle{Learning to {Predict} {Denotational}
  {Probabilities} {For} {Modeling} {Entailment}}. In
  \bibinfo{booktitle}{\emph{Proceedings of the 15th {Conference} of the
  {European} {Chapter} of the {Association} for {Computational} {Linguistics}:
  {Volume} 1, {Long} {Papers}}}, \bibfield{editor}{\bibinfo{person}{Mirella
  Lapata}, \bibinfo{person}{Phil Blunsom}, {and} \bibinfo{person}{Alexander
  Koller}} (Eds.). \bibinfo{publisher}{Association for Computational
  Linguistics}, \bibinfo{address}{Valencia, Spain}, \bibinfo{pages}{721--730}.
\newblock
\urldef\tempurl%
\url{https://aclanthology.org/E17-1068}
\showURL{%
\tempurl}


\bibitem[Lao et~al\mbox{.}(2011)]%
        {lao_random_2011}
\bibfield{author}{\bibinfo{person}{Ni Lao}, \bibinfo{person}{Tom Mitchell},
  {and} \bibinfo{person}{William~W. Cohen}.} \bibinfo{year}{2011}\natexlab{}.
\newblock \showarticletitle{Random {Walk} {Inference} and {Learning} in {A}
  {Large} {Scale} {Knowledge} {Base}}. In \bibinfo{booktitle}{\emph{Proceedings
  of the 2011 {Conference} on {Empirical} {Methods} in {Natural} {Language}
  {Processing}}}. \bibinfo{publisher}{Association for Computational
  Linguistics}, \bibinfo{address}{Edinburgh, Scotland, UK.},
  \bibinfo{pages}{529--539}.
\newblock
\urldef\tempurl%
\url{https://aclanthology.org/D11-1049}
\showURL{%
\tempurl}


\bibitem[Law et~al\mbox{.}(2019)]%
        {law_lorentzian_2019}
\bibfield{author}{\bibinfo{person}{Marc Law}, \bibinfo{person}{Renjie Liao},
  \bibinfo{person}{Jake Snell}, {and} \bibinfo{person}{Richard Zemel}.}
  \bibinfo{year}{2019}\natexlab{}.
\newblock \showarticletitle{Lorentzian {Distance} {Learning} for {Hyperbolic}
  {Representations}}. In \bibinfo{booktitle}{\emph{Proceedings of the 36th
  {International} {Conference} on {Machine} {Learning}}}.
  \bibinfo{publisher}{PMLR}, \bibinfo{pages}{3672--3681}.
\newblock
\urldef\tempurl%
\url{https://proceedings.mlr.press/v97/law19a.html}
\showURL{%
\tempurl}


\bibitem[Li et~al\mbox{.}(2019)]%
        {li_smoothing_2019}
\bibfield{author}{\bibinfo{person}{Xiang Li}, \bibinfo{person}{Luke Vilnis},
  \bibinfo{person}{Dongxu Zhang}, \bibinfo{person}{Michael Boratko}, {and}
  \bibinfo{person}{Andrew McCallum}.} \bibinfo{year}{2019}\natexlab{}.
\newblock \showarticletitle{Smoothing the {Geometry} of {Probabilistic} {Box}
  {Embeddings}}.
\newblock
\urldef\tempurl%
\url{https://openreview.net/forum?id=H1xSNiRcF7}
\showURL{%
\tempurl}


\bibitem[Li et~al\mbox{.}(2021)]%
        {li_discrete_2021}
\bibfield{author}{\bibinfo{person}{Yunqi Li}, \bibinfo{person}{Shuyuan Xu},
  \bibinfo{person}{Bo Liu}, \bibinfo{person}{Zuohui Fu},
  \bibinfo{person}{Shuchang Liu}, \bibinfo{person}{Xu Chen}, {and}
  \bibinfo{person}{Yongfeng Zhang}.} \bibinfo{year}{2021}\natexlab{}.
\newblock \bibinfo{title}{Discrete {Knowledge} {Graph} {Embedding} based on
  {Discrete} {Optimization}}.
\newblock
\newblock
\urldef\tempurl%
\url{https://doi.org/10.48550/arXiv.2101.04817}
\showDOI{\tempurl}
\newblock
\shownote{arXiv:2101.04817 [cs]}.


\bibitem[Lin et~al\mbox{.}(2015)]%
        {lin_learning_2015}
\bibfield{author}{\bibinfo{person}{Yankai Lin}, \bibinfo{person}{Zhiyuan Liu},
  \bibinfo{person}{Maosong Sun}, \bibinfo{person}{Yang Liu}, {and}
  \bibinfo{person}{Xuan Zhu}.} \bibinfo{year}{2015}\natexlab{}.
\newblock \showarticletitle{Learning {Entity} and {Relation} {Embeddings} for
  {Knowledge} {Graph} {Completion}}.
\newblock \bibinfo{journal}{\emph{Proceedings of the AAAI Conference on
  Artificial Intelligence}} \bibinfo{volume}{29}, \bibinfo{number}{1}
  (\bibinfo{date}{Feb.} \bibinfo{year}{2015}).
\newblock
\showISSN{2374-3468, 2159-5399}
\urldef\tempurl%
\url{https://doi.org/10.1609/aaai.v29i1.9491}
\showDOI{\tempurl}


\bibitem[Nakashole et~al\mbox{.}(2012)]%
        {nakashole_patty_2012}
\bibfield{author}{\bibinfo{person}{Ndapandula Nakashole},
  \bibinfo{person}{Gerhard Weikum}, {and} \bibinfo{person}{Fabian Suchanek}.}
  \bibinfo{year}{2012}\natexlab{}.
\newblock \showarticletitle{{PATTY}: a taxonomy of relational patterns with
  semantic types}. In \bibinfo{booktitle}{\emph{Proceedings of the 2012 {Joint}
  {Conference} on {Empirical} {Methods} in {Natural} {Language} {Processing}
  and {Computational} {Natural} {Language} {Learning}}}
  \emph{(\bibinfo{series}{{EMNLP}-{CoNLL} '12})}.
  \bibinfo{publisher}{Association for Computational Linguistics},
  \bibinfo{address}{USA}, \bibinfo{pages}{1135--1145}.
\newblock


\bibitem[Neelakantan et~al\mbox{.}(2015)]%
        {neelakantan_compositional_2015}
\bibfield{author}{\bibinfo{person}{Arvind Neelakantan},
  \bibinfo{person}{Benjamin Roth}, {and} \bibinfo{person}{Andrew McCallum}.}
  \bibinfo{year}{2015}\natexlab{}.
\newblock \bibinfo{title}{Compositional {Vector} {Space} {Models} for
  {Knowledge} {Base} {Completion}}.
\newblock
\newblock
\urldef\tempurl%
\url{https://doi.org/10.48550/arXiv.1504.06662}
\showDOI{\tempurl}
\newblock
\shownote{arXiv:1504.06662 [cs, stat]}.


\bibitem[Nguyen et~al\mbox{.}(2018)]%
        {nguyen_novel_2018}
\bibfield{author}{\bibinfo{person}{Dai~Quoc Nguyen}, \bibinfo{person}{Tu~Dinh
  Nguyen}, \bibinfo{person}{Dat~Quoc Nguyen}, {and} \bibinfo{person}{Dinh
  Phung}.} \bibinfo{year}{2018}\natexlab{}.
\newblock \showarticletitle{A {Novel} {Embedding} {Model} for {Knowledge}
  {Base} {Completion} {Based} on {Convolutional} {Neural} {Network}}. In
  \bibinfo{booktitle}{\emph{Proceedings of the 2018 {Conference} of the {North}
  {American} {Chapter} of the {Association} for {Computational} {Linguistics}:
  {Human} {Language} {Technologies}, {Volume} 2 ({Short} {Papers})}}.
  \bibinfo{pages}{327--333}.
\newblock
\urldef\tempurl%
\url{https://doi.org/10.18653/v1/N18-2053}
\showDOI{\tempurl}
\newblock
\shownote{arXiv:1712.02121 [cs]}.


\bibitem[Nickel and Kiela(2017)]%
        {nickel_poincare_2017}
\bibfield{author}{\bibinfo{person}{Maximillian Nickel} {and}
  \bibinfo{person}{Douwe Kiela}.} \bibinfo{year}{2017}\natexlab{}.
\newblock \showarticletitle{Poincaré {Embeddings} for {Learning}
  {Hierarchical} {Representations}}. In \bibinfo{booktitle}{\emph{Advances in
  {Neural} {Information} {Processing} {Systems}}}, Vol.~\bibinfo{volume}{30}.
  \bibinfo{publisher}{Curran Associates, Inc.}
\newblock
\urldef\tempurl%
\url{https://papers.nips.cc/paper/2017/hash/59dfa2df42d9e3d41f5b02bfc32229dd-Abstract.html}
\showURL{%
\tempurl}


\bibitem[Nickel and Kiela(2018)]%
        {nickel_learning_2018}
\bibfield{author}{\bibinfo{person}{Maximilian Nickel} {and}
  \bibinfo{person}{Douwe Kiela}.} \bibinfo{year}{2018}\natexlab{}.
\newblock \bibinfo{title}{Learning {Continuous} {Hierarchies} in the {Lorentz}
  {Model} of {Hyperbolic} {Geometry}}.
\newblock
\newblock
\urldef\tempurl%
\url{https://doi.org/10.48550/arXiv.1806.03417}
\showDOI{\tempurl}
\newblock
\shownote{arXiv:1806.03417 [cs, stat]}.


\bibitem[Nickel et~al\mbox{.}(2011)]%
        {nickel_three-way_2011}
\bibfield{author}{\bibinfo{person}{Maximilian Nickel}, \bibinfo{person}{Volker
  Tresp}, {and} \bibinfo{person}{Hans-Peter Kriegel}.}
  \bibinfo{year}{2011}\natexlab{}.
\newblock \showarticletitle{A three-way model for collective learning on
  multi-relational data}. In \bibinfo{booktitle}{\emph{Proceedings of the 28th
  {International} {Conference} on {International} {Conference} on {Machine}
  {Learning}}} \emph{(\bibinfo{series}{{ICML}'11})}.
  \bibinfo{publisher}{Omnipress}, \bibinfo{address}{Madison, WI, USA},
  \bibinfo{pages}{809--816}.
\newblock
\showISBNx{9781450306195}


\bibitem[Perozzi et~al\mbox{.}(2014)]%
        {perozzi_deepwalk_2014}
\bibfield{author}{\bibinfo{person}{Bryan Perozzi}, \bibinfo{person}{Rami
  Al-Rfou}, {and} \bibinfo{person}{Steven Skiena}.}
  \bibinfo{year}{2014}\natexlab{}.
\newblock \showarticletitle{{DeepWalk}: {Online} {Learning} of {Social}
  {Representations}}.
\newblock  (\bibinfo{date}{March} \bibinfo{year}{2014}).
\newblock
\urldef\tempurl%
\url{https://doi.org/10.1145/2623330.2623732}
\showDOI{\tempurl}


\bibitem[Ren et~al\mbox{.}(2020)]%
        {ren_query2box_2020}
\bibfield{author}{\bibinfo{person}{Hongyu Ren}, \bibinfo{person}{Weihua Hu},
  {and} \bibinfo{person}{Jure Leskovec}.} \bibinfo{year}{2020}\natexlab{}.
\newblock \bibinfo{title}{Query2box: {Reasoning} over {Knowledge} {Graphs} in
  {Vector} {Space} using {Box} {Embeddings}}.
\newblock
\newblock
\urldef\tempurl%
\url{https://doi.org/10.48550/arXiv.2002.05969}
\showDOI{\tempurl}
\newblock
\shownote{arXiv:2002.05969 [cs, stat]}.


\bibitem[Rudolph(2007)]%
        {rudolph_using_2007}
\bibfield{author}{\bibinfo{person}{Sebastian Rudolph}.}
  \bibinfo{year}{2007}\natexlab{}.
\newblock \showarticletitle{Using {FCA} for {Encoding} {Closure} {Operators}
  into {Neural} {Networks}}. In \bibinfo{booktitle}{\emph{Proceedings of the
  15th international conference on {Conceptual} {Structures}: {Knowledge}
  {Architectures} for {Smart} {Applications}}} \emph{(\bibinfo{series}{{ICCS}
  '07})}. \bibinfo{publisher}{Springer-Verlag}, \bibinfo{address}{Berlin,
  Heidelberg}, \bibinfo{pages}{321--332}.
\newblock
\showISBNx{9783540736806}
\urldef\tempurl%
\url{https://doi.org/10.1007/978-3-540-73681-3_24}
\showDOI{\tempurl}


\bibitem[Schlichtkrull et~al\mbox{.}(2017)]%
        {schlichtkrull_modeling_2017}
\bibfield{author}{\bibinfo{person}{Michael Schlichtkrull},
  \bibinfo{person}{Thomas~N. Kipf}, \bibinfo{person}{Peter Bloem},
  \bibinfo{person}{Rianne van~den Berg}, \bibinfo{person}{Ivan Titov}, {and}
  \bibinfo{person}{Max Welling}.} \bibinfo{year}{2017}\natexlab{}.
\newblock \bibinfo{title}{Modeling {Relational} {Data} with {Graph}
  {Convolutional} {Networks}}.
\newblock
\newblock
\urldef\tempurl%
\url{https://doi.org/10.48550/arXiv.1703.06103}
\showDOI{\tempurl}
\newblock
\shownote{arXiv:1703.06103 [cs, stat]}.


\bibitem[Shang et~al\mbox{.}(2018)]%
        {shang_end--end_2018}
\bibfield{author}{\bibinfo{person}{Chao Shang}, \bibinfo{person}{Yun Tang},
  \bibinfo{person}{Jing Huang}, \bibinfo{person}{Jinbo Bi},
  \bibinfo{person}{Xiaodong He}, {and} \bibinfo{person}{Bowen Zhou}.}
  \bibinfo{year}{2018}\natexlab{}.
\newblock \bibinfo{title}{End-to-end {Structure}-{Aware} {Convolutional}
  {Networks} for {Knowledge} {Base} {Completion}}.
\newblock
\newblock
\urldef\tempurl%
\url{https://doi.org/10.48550/arXiv.1811.04441}
\showDOI{\tempurl}
\newblock
\shownote{arXiv:1811.04441 [cs]}.


\bibitem[Shwartz et~al\mbox{.}(2016)]%
        {shwartz_improving_2016}
\bibfield{author}{\bibinfo{person}{Vered Shwartz}, \bibinfo{person}{Yoav
  Goldberg}, {and} \bibinfo{person}{Ido Dagan}.}
  \bibinfo{year}{2016}\natexlab{}.
\newblock \bibinfo{title}{Improving {Hypernymy} {Detection} with an
  {Integrated} {Path}-based and {Distributional} {Method}}.
\newblock
\newblock
\urldef\tempurl%
\url{https://doi.org/10.48550/arXiv.1603.06076}
\showDOI{\tempurl}
\newblock
\shownote{arXiv:1603.06076 [cs]}.


\bibitem[Suchanek et~al\mbox{.}(2007)]%
        {suchanek_yago_2007}
\bibfield{author}{\bibinfo{person}{Fabian~M. Suchanek},
  \bibinfo{person}{Gjergji Kasneci}, {and} \bibinfo{person}{Gerhard Weikum}.}
  \bibinfo{year}{2007}\natexlab{}.
\newblock \showarticletitle{Yago: a core of semantic knowledge}. In
  \bibinfo{booktitle}{\emph{Proceedings of the 16th international conference on
  {World} {Wide} {Web}}} \emph{(\bibinfo{series}{{WWW} '07})}.
  \bibinfo{publisher}{Association for Computing Machinery},
  \bibinfo{address}{New York, NY, USA}, \bibinfo{pages}{697--706}.
\newblock
\showISBNx{9781595936547}
\urldef\tempurl%
\url{https://doi.org/10.1145/1242572.1242667}
\showDOI{\tempurl}


\bibitem[Sun et~al\mbox{.}(2019)]%
        {sun_rotate_2019}
\bibfield{author}{\bibinfo{person}{Zhiqing Sun}, \bibinfo{person}{Zhi-Hong
  Deng}, \bibinfo{person}{Jian-Yun Nie}, {and} \bibinfo{person}{Jian Tang}.}
  \bibinfo{year}{2019}\natexlab{}.
\newblock \bibinfo{title}{{RotatE}: {Knowledge} {Graph} {Embedding} by
  {Relational} {Rotation} in {Complex} {Space}}.
\newblock
\newblock
\urldef\tempurl%
\url{https://doi.org/10.48550/arXiv.1902.10197}
\showDOI{\tempurl}
\newblock
\shownote{arXiv:1902.10197 [cs, stat]}.


\bibitem[Tang et~al\mbox{.}(2015)]%
        {tang_line_2015}
\bibfield{author}{\bibinfo{person}{Jian Tang}, \bibinfo{person}{Meng Qu},
  \bibinfo{person}{Mingzhe Wang}, \bibinfo{person}{Ming Zhang},
  \bibinfo{person}{Jun Yan}, {and} \bibinfo{person}{Qiaozhu Mei}.}
  \bibinfo{year}{2015}\natexlab{}.
\newblock \showarticletitle{{LINE}: {Large}-scale {Information} {Network}
  {Embedding}}.
\newblock  (\bibinfo{date}{March} \bibinfo{year}{2015}).
\newblock
\urldef\tempurl%
\url{https://doi.org/10.1145/2736277.2741093}
\showDOI{\tempurl}


\bibitem[Trouillon et~al\mbox{.}(2016)]%
        {trouillon_complex_2016}
\bibfield{author}{\bibinfo{person}{Théo Trouillon}, \bibinfo{person}{Johannes
  Welbl}, \bibinfo{person}{Sebastian Riedel}, \bibinfo{person}{Éric Gaussier},
  {and} \bibinfo{person}{Guillaume Bouchard}.} \bibinfo{year}{2016}\natexlab{}.
\newblock \bibinfo{title}{Complex {Embeddings} for {Simple} {Link}
  {Prediction}}.
\newblock
\newblock
\urldef\tempurl%
\url{https://doi.org/10.48550/arXiv.1606.06357}
\showDOI{\tempurl}
\newblock
\shownote{arXiv:1606.06357 [cs, stat]}.


\bibitem[Vendrov et~al\mbox{.}(2015)]%
        {vendrov_order-embeddings_2015}
\bibfield{author}{\bibinfo{person}{Ivan Vendrov}, \bibinfo{person}{Ryan Kiros},
  \bibinfo{person}{Sanja Fidler}, {and} \bibinfo{person}{Raquel Urtasun}.}
  \bibinfo{year}{2015}\natexlab{}.
\newblock \showarticletitle{Order-{Embeddings} of {Images} and {Language}}.
\newblock  (\bibinfo{date}{Nov.} \bibinfo{year}{2015}).
\newblock
\urldef\tempurl%
\url{https://doi.org/10.48550/arXiv.1511.06361}
\showDOI{\tempurl}


\bibitem[Vilnis et~al\mbox{.}(2018)]%
        {vilnis_probabilistic_2018}
\bibfield{author}{\bibinfo{person}{Luke Vilnis}, \bibinfo{person}{Xiang Li},
  \bibinfo{person}{Shikhar Murty}, {and} \bibinfo{person}{Andrew McCallum}.}
  \bibinfo{year}{2018}\natexlab{}.
\newblock \bibinfo{title}{Probabilistic {Embedding} of {Knowledge} {Graphs}
  with {Box} {Lattice} {Measures}}.
\newblock
\newblock
\urldef\tempurl%
\url{http://arxiv.org/abs/1805.06627}
\showURL{%
\tempurl}
\newblock
\shownote{arXiv:1805.06627 [cs, stat]}.


\bibitem[Vinyals et~al\mbox{.}(2015)]%
        {vinyals_show_2015}
\bibfield{author}{\bibinfo{person}{Oriol Vinyals}, \bibinfo{person}{Alexander
  Toshev}, \bibinfo{person}{Samy Bengio}, {and} \bibinfo{person}{Dumitru
  Erhan}.} \bibinfo{year}{2015}\natexlab{}.
\newblock \bibinfo{title}{Show and {Tell}: {A} {Neural} {Image} {Caption}
  {Generator}}.
\newblock
\newblock
\urldef\tempurl%
\url{https://doi.org/10.48550/arXiv.1411.4555}
\showDOI{\tempurl}
\newblock
\shownote{arXiv:1411.4555 [cs]}.


\bibitem[Vrandecic(2012)]%
        {vrandecic_wikidata_2012}
\bibfield{author}{\bibinfo{person}{Denny Vrandecic}.}
  \bibinfo{year}{2012}\natexlab{}.
\newblock \showarticletitle{Wikidata: a new platform for collaborative data
  collection}. In \bibinfo{booktitle}{\emph{Proceedings of the 21st {World}
  {Wide} {Web} {Conference}, {WWW} 2012, {Lyon}, {France}, {April} 16-20, 2012
  ({Companion} {Volume})}}, \bibfield{editor}{\bibinfo{person}{Alain Mille},
  \bibinfo{person}{Fabien Gandon}, \bibinfo{person}{Jacques Misselis},
  \bibinfo{person}{Michael Rabinovich}, {and} \bibinfo{person}{Steffen Staab}}
  (Eds.). \bibinfo{publisher}{ACM}, \bibinfo{pages}{1063--1064}.
\newblock
\urldef\tempurl%
\url{https://doi.org/10.1145/2187980.2188242}
\showDOI{\tempurl}


\bibitem[Wang et~al\mbox{.}(2019)]%
        {wang_learning_2019}
\bibfield{author}{\bibinfo{person}{Meng Wang}, \bibinfo{person}{Haomin Shen},
  \bibinfo{person}{Sen Wang}, \bibinfo{person}{Lina Yao},
  \bibinfo{person}{Yinlin Jiang}, \bibinfo{person}{Guilin Qi}, {and}
  \bibinfo{person}{Yang Chen}.} \bibinfo{year}{2019}\natexlab{}.
\newblock \showarticletitle{Learning to {Hash} for {Efficient} {Search} {Over}
  {Incomplete} {Knowledge} {Graphs}}. In \bibinfo{booktitle}{\emph{2019 {IEEE}
  {International} {Conference} on {Data} {Mining} ({ICDM})}}.
  \bibinfo{publisher}{IEEE}, \bibinfo{address}{Beijing, China},
  \bibinfo{pages}{1360--1365}.
\newblock
\showISBNx{9781728146041}
\urldef\tempurl%
\url{https://doi.org/10.1109/ICDM.2019.00174}
\showDOI{\tempurl}


\bibitem[Wang et~al\mbox{.}(2014)]%
        {wang_knowledge_2014}
\bibfield{author}{\bibinfo{person}{Zhen Wang}, \bibinfo{person}{Jianwen Zhang},
  \bibinfo{person}{Jianlin Feng}, {and} \bibinfo{person}{Zheng Chen}.}
  \bibinfo{year}{2014}\natexlab{}.
\newblock \showarticletitle{Knowledge {Graph} {Embedding} by {Translating} on
  {Hyperplanes}}.
\newblock \bibinfo{journal}{\emph{Proceedings of the AAAI Conference on
  Artificial Intelligence}} \bibinfo{volume}{28}, \bibinfo{number}{1}
  (\bibinfo{date}{June} \bibinfo{year}{2014}).
\newblock
\showISSN{2374-3468}
\urldef\tempurl%
\url{https://doi.org/10.1609/aaai.v28i1.8870}
\showDOI{\tempurl}


\bibitem[Wu et~al\mbox{.}(2012)]%
        {wu_probase_2012}
\bibfield{author}{\bibinfo{person}{Wentao Wu}, \bibinfo{person}{Hongsong Li},
  \bibinfo{person}{Haixun Wang}, {and} \bibinfo{person}{Kenny~Q. Zhu}.}
  \bibinfo{year}{2012}\natexlab{}.
\newblock \showarticletitle{Probase: a probabilistic taxonomy for text
  understanding}. In \bibinfo{booktitle}{\emph{Proceedings of the 2012 {ACM}
  {SIGMOD} {International} {Conference} on {Management} of {Data}}}
  \emph{(\bibinfo{series}{{SIGMOD} '12})}. \bibinfo{publisher}{Association for
  Computing Machinery}, \bibinfo{address}{New York, NY, USA},
  \bibinfo{pages}{481--492}.
\newblock
\showISBNx{9781450312479}
\urldef\tempurl%
\url{https://doi.org/10.1145/2213836.2213891}
\showDOI{\tempurl}


\bibitem[Yao et~al\mbox{.}(2019)]%
        {yao_kg-bert_2019}
\bibfield{author}{\bibinfo{person}{Liang Yao}, \bibinfo{person}{Chengsheng
  Mao}, {and} \bibinfo{person}{Yuan Luo}.} \bibinfo{year}{2019}\natexlab{}.
\newblock \bibinfo{title}{{KG}-{BERT}: {BERT} for {Knowledge} {Graph}
  {Completion}}.
\newblock
\newblock
\urldef\tempurl%
\url{https://doi.org/10.48550/arXiv.1909.03193}
\showDOI{\tempurl}
\newblock
\shownote{arXiv:1909.03193 [cs]}.


\bibitem[Zhang et~al\mbox{.}(2022)]%
        {zhang_modeling_2022}
\bibfield{author}{\bibinfo{person}{Dongxu Zhang}, \bibinfo{person}{Michael
  Boratko}, \bibinfo{person}{Cameron Musco}, {and} \bibinfo{person}{Andrew
  McCallum}.} \bibinfo{year}{2022}\natexlab{}.
\newblock \showarticletitle{Modeling {Transitivity} and {Cyclicity} in
  {Directed} {Graphs} via {Binary} {Code} {Box} {Embeddings}}.
\newblock \bibinfo{journal}{\emph{Advances in Neural Information Processing
  Systems}}  \bibinfo{volume}{35} (\bibinfo{date}{Dec.} \bibinfo{year}{2022}),
  \bibinfo{pages}{10587--10599}.
\newblock
\urldef\tempurl%
\url{https://papers.nips.cc/paper_files/paper/2022/hash/44a1f18afd6d5cc34d7e5c3d8a80f63b-Abstract-Conference.html}
\showURL{%
\tempurl}


\bibitem[Zhang and Saab(2021)]%
        {zhang_faster_2021}
\bibfield{author}{\bibinfo{person}{Jinjie Zhang} {and} \bibinfo{person}{Rayan
  Saab}.} \bibinfo{year}{2021}\natexlab{}.
\newblock \bibinfo{title}{Faster {Binary} {Embeddings} for {Preserving}
  {Euclidean} {Distances}}.
\newblock
\newblock
\urldef\tempurl%
\url{https://doi.org/10.48550/arXiv.2010.00712}
\showDOI{\tempurl}
\newblock
\shownote{arXiv:2010.00712 [cs, math, stat]}.


\end{thebibliography}

\appendix

\section{Proof of \name's Convergence to Local Optimal Solution \label{sec:proof of convergence}} 

\begin{lemma}
    When bias $b_\ell = 0$, for any word $a$, if the $j$'th bit in the binary representation vector $\mathbf{a}$ is updated by \name's probabilistic flipping (keeping the remaining bits the same), the sub-component of the positive loss corresponding to pairs $(a,\cdot)$ whose first element is $a$ decreases.
\end{lemma}

\begin{proof}
Let $P_{a,\cdot}$ be the set of positive data instances $(a,b)$ where the first entity is the given $a$, and let $\mathbf{a}$ be the $d$-dimensional embedding vector of $a$. Define $L(\mathbf{a})$ as the component of loss function associated with $a$. Suppose in an iteration, we randomly flip bit $j$ of $\mathbf{a}$.
To compute this probability, \text{\name} computes the $\Delta_{\mathbf{a}_j} Loss_P$, which is $L(\mathbf{a}) - L(\hat{\mathbf{a}})$ , where $\hat{\mathbf{a}}$ is the same as $\mathbf{a}$ except that the bit value in the $j$'th position is different.
(Recall that we define our gradient to be \emph{positive} when flipping bit $j$ improves our model, thus \emph{decreasing} the loss function.)
Based on Eq. \ref{eqn:grad-pos-a}, this gradient value is $+1$ only for the case when a constraint $\mathbf {a}_j \rightarrow \mathbf{b}_j$ is violated (where $b$ is the other element in a training pair), i.e. $\mathbf{a}_j = 0$ but $\mathbf{b}_j = 1$ (see the 3rd column of Table \ref{tab:grad-positive}).
For the other three choices of $\mathbf a_j$ and $\mathbf b_j$, $(0,0), (1,0), (1,1)$, the contribution to gradient value is 0, 0, and $-1$ respectively. Thus, the gradient will only be positive when $\mathbf{a}_j = 0$ (and $\mathbf{b}_j = 1$ for some $b$).
Using Line 7 of Algorithm \ref{alg:training} for $b_\ell = 0$, the flip probability can only be positive in the $\mathbf{a}_j = 0$ case, and with the flip the loss function value decreases by $k\alpha$ (through Eq. \ref{eqn:loss-p}), where $k$ is the number of pairs in $S$ that violate implication constraint with $a$ in the left side.
 In all scenarios, the loss value decreases in the successive iteration.
\end{proof} 

\begin{lemma}
    When bias $b_\ell = 0$, for any word $b$, if the $j$'th bit in the binary representation vector of $b$ is updated by \name's probabilistic flipping (keeping the remaining bits the same), the sub-component of the loss function corresponding to positive pairs $(\cdot, b)$ decreases.
\end{lemma}
\begin{proof}
The proof is identical to the proof of Lemma \ref{lem:flip-pos-a}, except that we use gradient value in Eq. \ref{eqn:grad-pos-b} instead of Eq. \ref{eqn:grad-pos-a}. In this case also when only one position of $b$'s embedding vector is flipped probabilistically, the loss function value decreases. 
\end{proof}

\begin{lemma}
    When bias $b_\ell = 0$, given a collection of negative data instances, say, $(a',b')$, if the $j$'th bit in the vectors of $a'$ or $b'$ independently (not simultaneously) is updated by \name's probabilistic flipping (keeping the remaining bits same), the loss function value decreases or remains the same in the successive iteration.
\end{lemma}
\begin{proof}
The proof is identical to proof of Lemma 1, except that we use gradient value in Eq. \ref{eqn:grad-neg-a} (or Eq. \ref{eqn:grad-neg-a-j}) for the case of $a'$, and gradient value of Eq. \ref{eqn:grad-neg-b} (or Eq. \ref{eqn:grad-neg-b-j}) for $b'$, and the loss function value decreases through Eq. \ref{eqn:loss-n}.
\end{proof}
These proofs also apply if $r_\ell \alpha \geq b_\ell > 0$ and $r_\ell \beta \geq b_\ell > 0$. In that case, we can flip a bit with zero gradient. Such flips do not immediately increase or decrease the loss function; however, they can allow {\name} to improve from a weak local optimum. In our experiments, $r_\ell \alpha$ and $r_\ell \beta$ are much larger than $b_\ell$, so our algorithm prioritizes actual improvements over zero-gradient steps.

\begin{theorem}
When bias $b_l = 0$, if Line 8 of Algorithm \ref{alg:training} is executed sequentially for each index value j for each of the entities, \text{\name} reaches a local optimal solution considering a 1-hamming distance neighborhood.
\end{theorem}
\begin{proof}Using earlier Lemmas, each bit flipping in any embedding vector of any of the entities, either decreases the loss function or keeps it the same. Each sub-component of the gradient is exactly the negative of the change in that sub-component of the loss when the corresponding bit is flipped; hence, if $\Delta[a,j] > 0$, then $Loss$ will decrease by exactly $\Delta[a,j]$ when $\mathbf{a}_j$ is flipped, and if $\Delta[a,j] \leq 0$, then $\mathrm{FlipProb}(a,j)$ will be 0 and $\mathbf{a}_j$ will not be flipped.

When Line 8 of Algorithm 2 is executed sequentially for each index $j$ (only one change in one iteration) for each of the entities, the loss function value monotonically decreases in each successive iteration. Because the space of embeddings is finite, eventually {\name} will reach a point where no single bit flip improves the value of Loss. Now, if the bias $b_l = 0$, for each entity, the probability of bit-flipping for each index is computed to be 0 (by Line 7 in Algorithm \ref{alg:training}), so none of the embeddings change any further and \text{\name} reaches a local optimal solution considering a 1-hamming distance neighborhood. In other words, for every entity $a$, if we change any single bit of $a$'s embedding, the original embedding of $a$ is guaranteed to be at least as good as the changed embedding.
\end{proof}

When we change only one bit at a time, keeping everything else the same (as in the proof), our optimization algorithm becomes a greedy hill climbing algorithm. However, this would make \text{\name} extremely slow to converge, and it may get stuck in a bad local optimal solution. Thus, we allow all bits to change simultaneously, so it behaves like gradient descent: Suppose $\theta$ is an embedding vector of an entity and $L(\theta)$ is the component of loss function associated with this entity. For minimizing $L(\theta)$, at each iteration, a hill climbing method would adjust a single element in $\theta$ to decrease $L(\theta)$; on the other hand, gradient descent will adjust all values in $\theta$ in each iteration by using $\theta^{new} = \theta^{old} -\alpha\Delta_\theta L(\theta^{old})$. During early iterations, \text{\name } works like gradient descent, but as iteration progresses, it behaves more like a hill climbing method as gradient values for most bit positions decrease, causing fewer bits to flip. 

\section{Pseudo-code of \name \label{sec:pseudo-code}} 

Pseudo-code of \name\ is shown in Algorithm \ref{alg:training}. We initialize the embedding matrix with 0 vector. The optimization scheme loops for at most $T$ epochs (for loop in Line 4-17), possibly updating bit vectors of each vocabulary word in each iteration by flipping bits with probability in proportion to gradient computed through Algorithm \ref{alg:gradient}. The F1-score metric, defined as 
$\frac{2*TP}{2*TP+FP+FN}$ is computed at the end of each epochs and best F1-score is recorded. We exit early if no improvement is achieved over a sequence of epochs. 

\begin{algorithm}[htb]
    \caption{Gradient Computation}
    \label{alg:gradient}
    \begin{algorithmic}[1]
        \small
        \Require Zero-one Embedding Matrix $B$ of size $n \times d$ initialized with all 0; positive Is-A relation set $P = \{ (a^i, b^i) \}_{i=1}^{m}$; negative set $N = \{(a'^i, b'^i)\}_{i=1}^{m'}$; positive and negative sample weights $\alpha,\beta$
        \State $\Delta^+ \leftarrow$ zero matrix, same size as $B$
        \State $\Delta^- \leftarrow$ zero matrix, same size as $B$
        \For {$(a,b) \in P$} \Comment{$*$ is element-wise product}
            \State $\Delta^+[a,:] \leftarrow \Delta^+[a,:] + B[b,:] * (1 - 2B[a,:])$
            \State $\Delta^+[b,:] \leftarrow \Delta^+[b,:] + (1 - B[a,:]) * (2B[b,:] - 1)$
        \EndFor
        \For {$(a',b') \in N$}
            \State $\mathbf G \leftarrow B[b',:] * (1 - B[a',:])$ \Comment{``good'' bit pairs (a vector)}
            \If { $\sum_j \mathbf G_j = 0$ }  \Comment{false positive, flip something}
                \State $\Delta^-[a',:] \leftarrow \Delta^-[a',:] + B[a',:] * B[b',:]$
                \State $\Delta^-[b',:] \leftarrow \Delta^-[b',:] + (1 - B[a',:]) * (1 - B[b',:])$
            \ElsIf { $\sum_j \mathbf G_j = 1$ }  \Comment{close to being wrong, so protect}
                \State $\Delta^-[a',:] \leftarrow \Delta^-[a',:] - \mathbf G$ \Comment{note only one element of $\mathbf G$ is 1}
                \State $\Delta^-[b',:] \leftarrow \Delta^-[b',:] - \mathbf G$
            \EndIf
        \EndFor
        \State \Return $\Delta := \alpha \Delta^+ + \beta \Delta^-$
    \end{algorithmic}
\end{algorithm}

\begin{algorithm}[htb]
    \caption{Training Algorithm}
    \label{alg:training}
    \begin{algorithmic}[1]
        \small
        \Require Word list $W = (w_1, \dots, w_n)$; Dimension $d$; Positive training set $P = \{ (a^i, b^i) \}_{i=1}^{m}$; validation sets $VP, VN$; gradient weights $\alpha,\beta$, learning params $r_\ell, b_\ell$, negative sample multiplier $n^-$ (must be even); maximum epochs $T$, early stop width $\omega$
        \State $B \leftarrow$ zero matrix of size $|W| \times d$
        \State $f1 \leftarrow$ empty list
        \State $(BestEmbedding, BestF1) \leftarrow (B, 0)$
        \For {$t=1$ to $T$}
            \State $N \leftarrow$ negative samples (Section \ref{sec:training-algorithm})
            \State $\Delta \leftarrow $ gradient from Algorithm \ref{alg:gradient}
            \State $X \leftarrow \max\left\{ 0, \frac 1 2 \tanh(2(r_\ell \Delta + b_\ell)) \right\}$ \Comment{flip probabilities}
            \State Flip each bit $B[w,j]$ with (independent) probability $X[w,j]$
            \State $f1 \leftarrow \text{f1-score}(\mathrm{Evaluate}(B, VP, VN))$
            \If{$f1 > BestF1$}
                \State $(BestEmbedding, BestF1) \leftarrow (B, f1)$
            \EndIf
            \State Append $f1$ to list $F1$
            \If{ $\mathrm{mean} (F1[t-2\omega+1...t-\omega]) \geq \mathrm{mean} (F1[t-\omega+1...t])$}
                \State Exit Loop \Comment{Early Exit Criterion if no improvement}
            \EndIf
        \EndFor
        \State \Return $BestEmbedding$
    \end{algorithmic}
\end{algorithm}

\subsection{Computational Complexity Analysis \label{sec:computational_complexity}}

During Gradient Computation we possibly update d-dimensional bit vectors of each vocabulary word along each dimension for a pair (positive or negative). So, for a vocabulary size of $n$, and list $P$ positive pairs, we randomly generate $n^- |P|$ negative samples (where $n^-$ is a hyper-parameter), and so Gradient Computation complexity is $O(nd \, (1+n^-)|P|)$. We repeat this process for each epoch, so for $T$ epochs, the overall algorithm computational complexity is linear in each variable: $O(ndTn^-|P|)$.

\section{Experimental Setup \label{sec:experimental_setup}}

\subsection{Edge distribution for all datasets \label{subsec:edge_distribution}}

For each dataset we show the direct and indirect edge counts in table \ref{tab:data_statistics} and edge distribution percentage in Figure \ref{fig:edge_distribution}. 

\begin{table}[t]
\small
  \caption{Edge distribution for all datasets}
  \vspace{-3mm}
  \label{tab:data_statistics}
  \centering
  \begin{tabular}{c|l|l|l}
    \toprule
     & \multicolumn{3}{c}{\textbf{Edge Counts}}  \\
    Dataset & Direct  & Indirect  & Full Transitive \\
     &  & (Transitive) & Closure \\
    &  &  & (Direct + Transitive) \\
    \midrule
    Medical & 2616 & 1692 & 4308 \\
    Music & 3920 & 2608 & 6528 \\
    Shwartz Lex & 5566 & 7940 & 13506 \\
    Shwartz Random & 13740 & 42437 & 56177 \\
    Animals & 4051 & 25744 & 29795 \\
    WordNet Nouns & 84363 & 576764 & 661127 \\
  \bottomrule
\end{tabular}
\vspace{-3mm}
\end{table}

\begin{figure}
    \centering
    \includegraphics[width=0.3\textwidth]{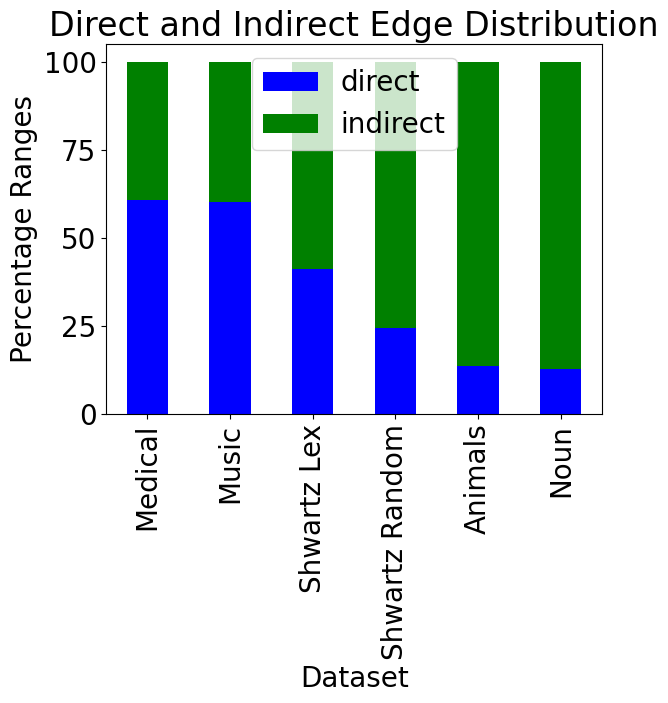}
    \vspace{-4mm}
    \caption{Distribution of Direct and Transitive (Indirect) edges for all datasets.}
    \label{fig:edge_distribution}
\end{figure}

\subsection{Hyperparameter Tuning \label{subsec:hyperparameter-tuning}}

For BINDER and the competing models, except Gumbel and T-Box, we exhaustively tuned dimensions $d={8, 32, 64, 128}$ while keeping patience at 500 (model will stop training if loss does not decrease in 500 iterations). Since Gumbel Box requires 2*d dimension and T-Box requires 4*d dimension for their representations, to be fair with other models, we tuned Gumbel Box for dimensions $d={4, 16, 32, 64}$ and T-Box for dimensions $d={2, 8, 16, 32}$ We report other hyperparameter ranges that we considered for tuning each model in table \ref{tab:hyperparameter_range} and the best hyperparameters for {\name } in table \ref{tab:best_hyperparameters}. All the models were run on a Tesla V100 GPU. 

\begin{table*}[t]
    \centering
    \caption{Hyper parameter range used for tuning each model. }

    \vspace{-3mm}
    \begin{tabular}{|c|l|}
        \toprule
        Model & Hyper parameter range \\
        \midrule
        Lorentzian distance & lorentzian $\alpha$: [0.01, 0.1]  \\
        Hyperbolic Entailment Cones & relation cone aperture scale: [0.01, 0.1, 1], eps bound: [0.01, 0.1] \\
        Box Embeddings & intersection temp: [0.01, 0.1, 1] and volume temp: [0.1, 1] \\
        BINDER & $\alpha$: [15, 20, 25, 30, 40, 100, 500, 5000, 25000, 50000] and $n^-$: [12, 32, 64, 128, 256, 512] \\
        \bottomrule
    \end{tabular}
    \label{tab:hyperparameter_range}
    \vspace{-1mm}
\end{table*}

\begin{table*}[t]
  \caption{Best hyper parameter configurations (dim,$\alpha$,$n^-$) for \name}
  \vspace{-3.5mm}
  \label{tab:best_hyperparameters}
  \centering
  \begin{tabular}{|c|l|l|l|l|l|l|}
    \toprule
    & \multicolumn{6}{c}{Dataset} \\
    Task & Medical & Music & Animals & Shwartz Lex & Shwartz Random & WordNet Nouns \\
    \midrule
    Representation & (128, 40, 512) & (128, 25, 128) & (128, 25, 256) & (128, 40, 512) & (128, 25, 512) & (128, 100, 256) \\
    TC 0\% & (128, 25000, 128) & (128, 25000, 128) & (128, 25000, 128) & (128, 25000, 128) & (128, 25000, 128) & (128, 25000, 128)\\
    TC 10\% & (128, 25000, 128) & (128, 10000, 128) & (128, 25000, 128) & (128, 25000, 128) & (128, 25000, 128) & (128, 25000, 128)\\
    TC 25\% & (128, 50000, 128) & (128, 25000, 128) & (128, 50000, 128) & (128, 50000, 128) & (128, 50000, 128) & (128, 50000, 12)\\
    TC 50\% & (128, 50000, 128) & (128, 50000, 128) & (128, 50000, 128) & (128, 50000, 128) & (128, 50000, 128) & (128, 50000, 12)\\
  \bottomrule
\end{tabular}
\vspace{-2.5mm}
\end{table*}

\subsection{Ablation Study\label{subsec:ablation}}
For an ablation study, we observed the effect on our model's $F1$-score for the test data by removing $\beta$ and $b_\ell$ separately while keeping dimension fixed at 128, and $\alpha$ = 25 for representation and $\alpha$ = 25000 for $0\%$ TC link prediction. The default values considered for $b_\ell$ and $\beta$ are 0.01 and 32 respectively. For the representation task on Animals dataset, setting $\beta = 0$ gave an $F1$-score of $44.07 \pm 1.77 \%$, demonstrating a significant drop from our experimental results, whereas setting $b_\ell = 0$ we did not see any significant effect on the $F1$-score ($96.46 \pm 1.34 \%$). For the $0\%$ TC prediction task on Animals dataset, when we set $\beta = 0$, we got an $F1$-score of $60.49 \pm 0.55 \%$ which was lower than our experimental results. Whereas setting $b_\ell=0$ we did not see any significant effect on the $F1$-score ($98.73 \pm 0.45$).

\subsection{Robustness of \name\label{subsec:robustness}}

\name\ has only a few hyper-parameters, which we tuned to determine the effect of each hyper parameter on model F1-score. In this experiment, we validate how the performance of {\name} is affected by the choice of hyper-parameter values. We show plots for hyper parameter tuning study on the \emph{Animals} dataset for representation and link prediction (0\% transitive) task in Figure \ref{fig:animals_robustness}. We fix $\beta$ at 10, learning rate $r_\ell$ at 0.008 and learning bias at 0.01 and tune dim, positive weight $\alpha$ and negative sample multiplier $n^-$. We ran each experiment for 2500 iterations.

\begin{figure}[ht]
    \centering
    Representation Experiment $F1$-score
    
    \includegraphics[width=0.46\textwidth]{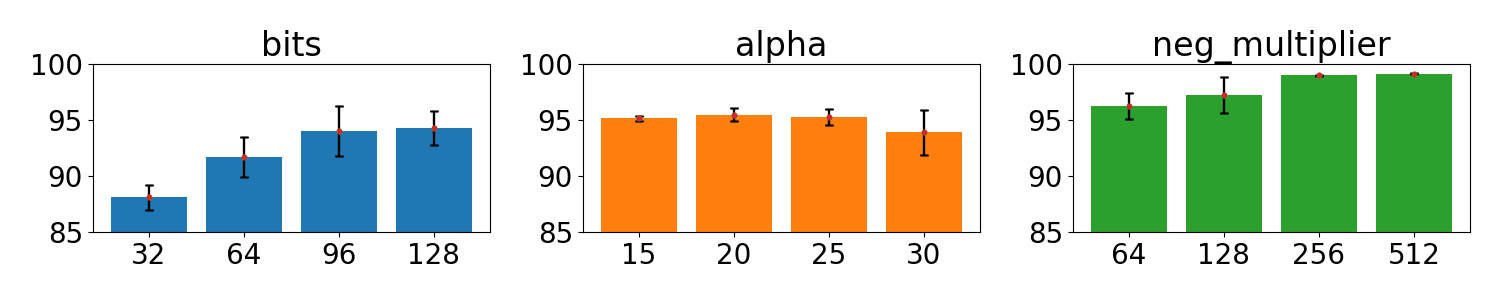}
    \vspace{1mm}

    0\% Transitive Closure Prediction Experiment $F1$-score
    \includegraphics[width=0.46\textwidth]{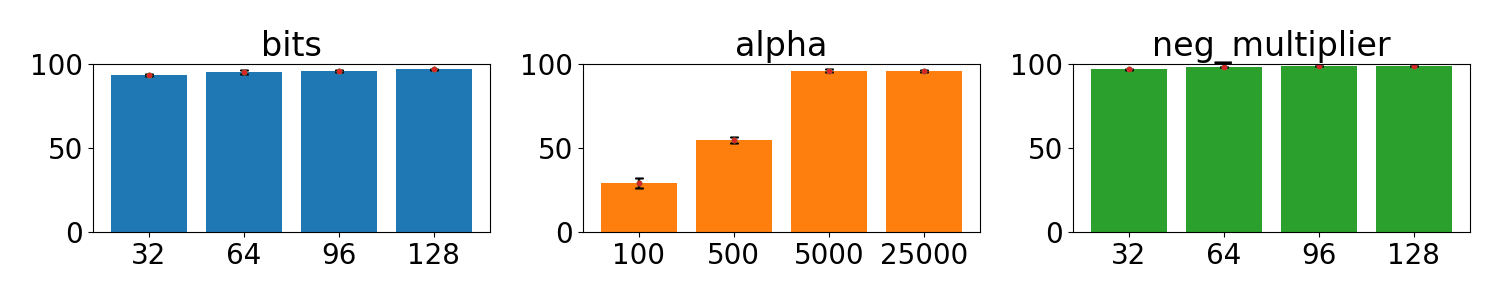}

    \vspace{-4mm}
    
    \caption{\name\ robustness experiment for Representation and 0\% Transitive Closure on Animals dataset.}
    \label{fig:animals_robustness}
\end{figure}

For representation task, to see the effect of dimension, we fix $\alpha$ at 30 and $n^-$ at 32 and change bits in the range $d \in \{32, 64, 96, 128\}$. We see an increasing trend in $F1$-score with optimum value at bits 128. This is intuitive as representation task benefits from large number of bits. To see the effect of $\alpha$ we fix bits at 128 and $n^-$ at 32 and change $\alpha \in \{15, 20, 25, 30\}$. We see $F1$-score reaches maximum at $\alpha=20$ and decreases after 25. For $n^-$, we fix bits at 128 and $\alpha$ at 25 and change $n^- \in \{64, 128, 256, 512\}$. As we increase $n^-$ we see increase in $F1$-score and also sharp decrease in error margin. Based on this experiment we decided on the optimal hyperparameter configuration to be (128, 25, 256).

For 0\% transitive closure task, we run similar experiements. For $\alpha = 25000$, $n^- = 32$, and $d \in \{32, 64, 96, 128\}$, we see a slightly increasing trend in $F1$-score with optimum value at 128 bits. So, transitive closure prediction, which depends on the generalization of the model, is not heavily influenced by increasing number of bits like the representation task.
For $d = 128$, $n^- = 32$, and $\alpha \in \{100, 500, 5000, 25000\}$, we see $F1$-score reaches maximum at $\alpha=25000$.
For $d = 128$, $\alpha = 25000$, and $n^- \in \{32, 64, 96, 128\}$, we see a very slight increase in $F1$-score as we increase $n^-$, with almost 100\% $F1$-score at $n^-=128$. Based on this experiment we decided on the optimal hyperparameters to be (128, 25000, 128). Similarly, we performed hyperparameter tuning for all datasets and all tasks and reported the best hyper parameters in Table \ref{tab:best_hyperparameters} in Appendix \ref{subsec:hyperparameter-tuning}.

\subsection{\name\ Model Convergence Results \label{sec:convergence}}

We run our model with a large number of iterations to minimize loss and thus maximize the $F1$-score. We show train loss curves for both the representation and the $0\%$ TC link prediction tasks respectively in the first and second graphs in Figure~\ref{fig:nouns_convergence}. We see an initial spike (for representation task) and turbulence (for TC link prediction task); we attribute the spike to the way loss is calculated when embeddings are initialized as 0: $Loss^+ = 0$ because there are no $(0,1)$ bit pairs \emph{anywhere} in the embedding, so the loss comes entirely from $Loss^-$. The turbulence is likely due to the smaller size of the positive, and hence negative, train data in the 0\% TC link prediction task.

For representation task our model achieves 99\% validation $F1$-score in less than 50 iterations on the \emph{WordNet Nouns} dataset, as shown in the third graph in Figure~\ref{fig:nouns_convergence}. For TC link prediction task the model attains high $F1$-score very quickly, and it continues to improve steadily, as shown in the fourth graph in Figure~\ref{fig:nouns_convergence}. 

\begin{figure}[ht]
    \centering

    \includegraphics[width=0.23\textwidth]{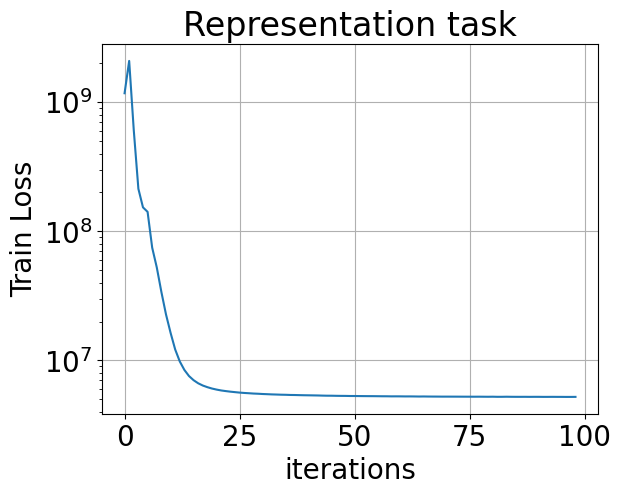}
    \includegraphics[width=0.23\textwidth]{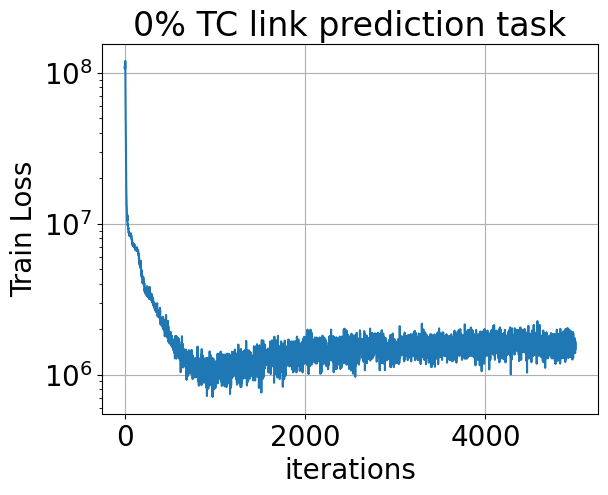}

    \includegraphics[width=0.23\textwidth]{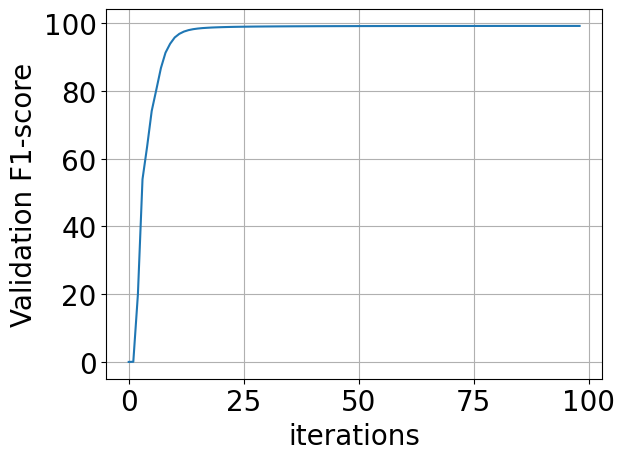}
    \includegraphics[width=0.23\textwidth]{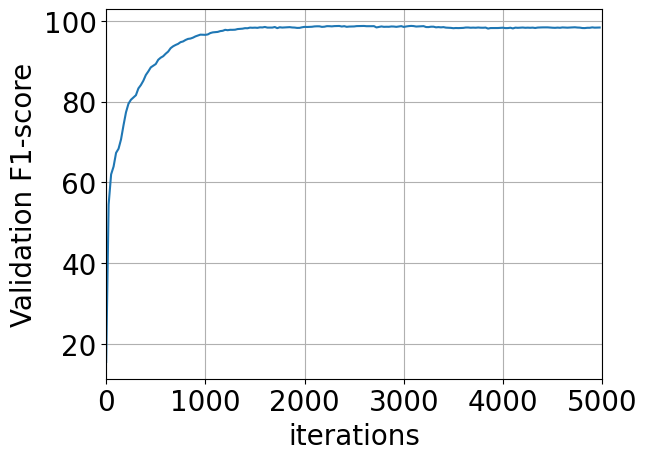}

    \caption{ Top-left figure: Train loss of the first 100 iterations for the representation task, Top-right figure: Train loss of the first 5000 iterations for the 0\% TC link prediction task,
    Bottom-left figure: Validation $F1$-score of the first 100 iterations for the representation task, Bottom-right figure: Validation $F1$-score of the first 5000 iterations for the 0\% TC link prediction task. All the experiments are performed on our largest \emph{WordNet Nouns} dataset.}
    \label{fig:nouns_convergence}
    \vspace{-1mm}
\end{figure}

\section{{\name} Binary Properties \label{sec:binary_properties}}
We claimed: ``Binary representation can immediately provide representation of concepts that can be obtained by logical operation over the given concept vectors''.
This claim is based on the intent-extent philosophy of Formal Concept Analysis (FCA). Formal concept analysis (FCA)~\cite{ganter_formal_1999}
provides a principled approach of deriving a concept hierarchy from a collection of objects and their attributes.
A formal context is a triple $(G,M,I)$, where $G$ is a set of objects, $M$ is a set of attributes, and $I \subset G \times M$ is a binary relation expressing which objects have which attributes.
Equivalently, a formal context can be represented as a $\{0,1\}$-matrix in which the rows corresponds to the objects and the columns corresponds to the attributes.

For the case of {\name}, we can think of the embeddings as creating the objects and attributes for FCA; each dimension of the embedding vectors is one attribute. For any given word $a$ with embedding vector $\mathbf{a}$, we can construct the Formal Concept $FC(a)$ whose \emph{objects} (extent) are $a$ and all its predicted hyponyms, and whose \emph{attributes} (intent) are the columns $j$ such that $\mathbf{a}_j = 1$.
The \emph{meet} of two concepts $FC(a)$ and $FC(b)$ contains all objects that are hyponyms of both $a$ and $b$, and all attributes possessed by \emph{at least one} of $a,b$; this corresponds to bitwise OR of embedding vectors. Dually, the \emph{join}'s intent is the set of attributes shared by \emph{both} words; this corresponds to bitwise AND.

For simplicity, we can identify corresponding words, embeddings, and formal concepts and simply define the \emph{meet} and \emph{join} of two words or two bit vectors to be another bit vector. Thus, the \emph{meet} operation on two individual entities corresponds to the binary OR operation on embeddings, while \emph{join} corresponds to binary AND.
While \emph{meet} is usually self-explanatory, the \emph{join} operation makes the most sense when used ``locally'', i.e. joining two similar objects gives a more meaningful result than joining unrelated concepts.

\begin{table}[h]
    \centering
    \begin{tabular}{l|c c c c c c}
    \toprule
        \textbf{Word} & \multicolumn{6}{c}{\textbf{Embedding}} \\
        \midrule
        flying       & 1 & 0 & 0 & 0 & 0 & 0 \\
        vehicle      & 0 & 0 & 1 & 0 & 0 & 0 \\
        airplane     & 1 & 1 & 1 & 0 & 0 & 0 \\
        helicopter   & 1 & 0 & 1 & 1 & 0 & 0 \\
        shoe         & 0 & 0 & 0 & 0 & 1 & 0 \\  
        mens-shoe    & 0 & 0 & 0 & 1 & 1 & 0 \\
        womens-shoe  & 0 & 0 & 0 & 0 & 1 & 1 \\
    \bottomrule
    \end{tabular}
    \caption{Example embedding of a small selection of instructive words. Other entities, like ``flying-vehicle'', can be inferred from these embeddings.}
    \label{tab:binder_example}
    \vspace{-5mm}
\end{table}
As an example, consider the embeddings in Table \ref{tab:binder_example}.  $\emph{meet}(\emph{flying,}$ $\emph{vehicle})$ is the vector $(1,0,1,0,0,0)$, which is easily interpreted as \emph{flying-vehicle}, a hypernym of \emph{airplane} and \emph{helicopter}.
It is harder to interpret $\emph{join}(\emph{flying}, \emph{vehicle})$, since the two words are rather unrelated; we get the zero vector, not an embedding of ``flying OR vehicle''. (In a more complete dataset, we might get something like ``thing that moves''.) A more sensible, ``local'' join operation is $\emph{join}(\emph{mens-shoe}, \emph{womens-shoe}) = \emph{shoe}$, since \emph{mens-shoe} and \emph{womens-shoe} are direct hyponyms of \emph{shoe}; we can see this in Table \ref{tab:binder_example} as $(0, 0, 0, 1, 1, 0) \cap (0, 0, 0, 0, 1, 1) = (0, 0, 0, 0, 1, 0)$.
{\name}, in the simplest way possible, allows one to infer entities like \emph{flying-vehicle} or confirm positioning of existing entities like \emph{shoe}.
Note that we can construct ``impossible'' meets like $\emph{meet}(shoe, vehicle)$, which results in the vector $\mathbf{z} = (0,0,1,0,1,0)$. The fact that $\mathbf{z}$ has no hyponyms indicates that nothing is both a \emph{vehicle} and a \emph{shoe}.

{\name} is, of course, not the only embedding method with intrinsic meet and join operations. Order Embeddings has very simple built-in \emph{meet}s and \emph{join}s, although {\name}'s bitwise AND/OR is arguably slightly simpler than OE's element-wise min/max. Box embeddings also support meets and joins, but the probabilistic nature makes the interpretation of the boxes more difficult. In contrast, Poincar\'e embeddings do not possess them at all, and in HEC, some pairs of cones fail to have meets and/or joins. In essence, we believe that bit vectors are \emph{the simplest possible} embeddings that naturally include the meet and join operations. {\name} embeddings can be directly interpreted, at least in theory, as sets of attributes or properties that must be inherited by all sub-concepts.

Of the three basic binary operations, {\name} handles one (AND) extremely well and can attach a lesser meaning to another (OR). In a very limited way, we can also consider the final operation, bitwise NOT. While the complement of a {\name} embedding is \emph{not} an embedding of its negation, we can still assign meaning to the complement vector. Consider the entity \emph{helicopter} from Table \ref{tab:binder_example}, embedded as $\mathbf{u} = (1,0,1,1,0,0)$.
Its complement is $1 - \mathbf{u} = (0,1,0,0,1,1)$, which has no hyponyms. What we \emph{can} say about $1 - \mathbf{u}$ is that no nonzero vector can possibly be a hypernym of both it and $\mathbf{u}$. Since $\emph{shoe} = (0,0,0,0,1,0)$, for instance, is a hypernym of $1 - \mathbf{u}$, we know immediately that \emph{helicopter} and \emph{shoe} have no \emph{attributes} in common and hence no common hypernyms (except the zero vector, which is a hypernym of everything and embeds the most generic concept, like ``entity'' or ``animal''). We can work similarly in the opposite direction; if $1 - \mathbf{u}$ did represent an entity (or had hyponyms), no hyponym of $1 - \mathbf{u}$ can be a hyponym of $\mathbf{u}$ except the bottom vector $(1,1,1,1,1,1)$. Since most natural hierarchies will have mutually exclusive concepts (like \emph{vehicle} and \emph{shoe}), the all-1 vector will not embed anything, and so in practice, $\mathbf{u}$ and $1-\mathbf{u}$ will have no \emph{objects} in common. This is not truly a ``negation'' concept, but it does have a useful meaning.

\end{document}